\documentclass{article}

\usepackage{microtype}
\usepackage{graphicx}
\usepackage{booktabs} % for professional tables

\usepackage{amsthm}
\usepackage{amsmath}
\usepackage{amssymb}
\usepackage{mathrsfs}
\usepackage{enumitem}
\usepackage{indentfirst}
\usepackage{dsfont}
\usepackage{breqn}
\usepackage[ruled,vlined]{algorithm2e} 
\usepackage[colorinlistoftodos]{todonotes}
\usepackage[colorlinks=true, allcolors=blue]{hyperref}
\usepackage{shortcuts}
\usepackage{makecell}

\usepackage{float}
\usepackage[caption = false]{subfig}
\usepackage{graphicx}

\usepackage{todonotes}

\usepackage{lipsum}

\usepackage{lipsum}

\usepackage{lipsum}

\DeclareMathOperator*{\vect}{vec}

\usepackage{hyperref}

\usepackage{xr}
\makeatletter
\newcommand*{\addFileDependency}[1]{% argument=file name and extension
  \typeout{(#1)}
  \@addtofilelist{#1}
  \IfFileExists{#1}{}{\typeout{No file #1.}}
}
\makeatother

\newcommand*{\myexternaldocument}[1]{%
    \externaldocument[supp-]{#1}%
    \addFileDependency{#1.tex}%
    \addFileDependency{#1.aux}%
}
\myexternaldocument{supplement}

\usepackage[accepted]{icml2021}

\icmltitlerunning{Best Arm Identification in Graphical Bilinear Bandits}

\begin{document}

\newtheorem{theorem}{Theorem}[section]
\newtheorem{lemma}[theorem]{Lemma}
\newtheorem{prop}[theorem]{Proposition}
\newtheorem{corollary}[theorem]{Corollary}
\newtheorem{definition}{Definition}[section]
\newtheorem{conj}{Conjecture}[section]
\newtheorem{example}{Example}[section]
\newtheorem{case}{Case}
\newtheorem{remark}{Remark}[section]

\twocolumn[
\icmltitle{Best Arm Identification in Graphical Bilinear Bandits}

% It is OKAY to include author information, even for blind
% submissions: the style file will automatically remove it for you
% unless you've provided the [accepted] option to the icml2021
% package.

% List of affiliations: The first argument should be a (short)
% identifier you will use later to specify author affiliations
% Academic affiliations should list Department, University, City, Region, Country
% Industry affiliations should list Company, City, Region, Country

% You can specify symbols, otherwise they are numbered in order.
% Ideally, you should not use this facility. Affiliations will be numbered
% in order of appearance and this is the preferred way.
\icmlsetsymbol{equal}{*}

\begin{icmlauthorlist}
\icmlauthor{Geovani Rizk}{PSL,Huawei}
\icmlauthor{Albert Thomas}{Huawei}
\icmlauthor{Igor Colin}{Huawei}
\icmlauthor{Rida Laraki}{PSL,Liverpool}
\icmlauthor{Yann Chevaleyre}{PSL}
\end{icmlauthorlist}

\icmlaffiliation{PSL}{PSL - Université Paris Dauphine, CNRS, LAMSADE, Paris, France}
\icmlaffiliation{Huawei}{Huawei Noah's Ark Lab}
\icmlaffiliation{Liverpool}{Liverpool University}

\icmlcorrespondingauthor{Geovani Rizk}{geovani.rizk@dauphine.psl.eu}

% You may provide any keywords that you
% find helpful for describing your paper; these are used to populate
% the "keywords" metadata in the PDF but will not be shown in the document
\icmlkeywords{Machine Learning, ICML}

\vskip 0.3in
]

% this must go after the closing bracket ] following \twocolumn[ ...

% This command actually creates the footnote in the first column
% listing the affiliations and the copyright notice.
% The command takes one argument, which is text to display at the start of the footnote.
% The \icmlEqualContribution command is standard text for equal contribution.
% Remove it (just {}) if you do not need this facility.

\printAffiliationsAndNotice{}  % leave blank if no need to mention equal contribution
% \printAffiliationsAndNotice{\icmlEqualContribution} % otherwise use the standard text.

\begin{abstract}
We introduce a new graphical bilinear bandit problem where a learner (or a \emph{central entity}) allocates arms to the nodes of a graph and observes for each edge a noisy bilinear reward representing the interaction between the two end nodes. We study the best arm identification problem in which the learner wants to find the graph allocation maximizing the sum of the bilinear rewards. By efficiently exploiting the geometry of this bandit problem, we propose a \emph{decentralized} allocation strategy based on random sampling with theoretical guarantees. In particular, we characterize the influence of the graph structure (e.g. star, complete or circle) on the convergence rate and propose empirical experiments that confirm this dependency.
\end{abstract}

\section{Introduction}
\label{sec:intro}
% \albert{verifier si on est coherent dans tout l'article: edge arms vs edge-arms, pareil pour node arms vs node-arms}

% The agents do not contribute independently to their common team objective and must coordinate (or be coordinated depending on the centralized or decentralized nature of the problem) to achieve the best result.
In many multi-agent systems the contribution of an agent to a common team objective is impacted by the behavior of the other agents. The agents must coordinate (or be coordinated) to achieve the best team performance.
Consider, for instance, the problem of configuring antennas of a wireless cellular network to obtain the best signal quality over the whole network \citep{siomina2006wireless}.
The signal quality of the region covered by a given antenna might be degraded by the behavior of its neighboring antennas due to an increase of interferences or bad user handovers.
% An antenna optimizing its own signal can be suboptimal for the whole network as its best configuration might degrade the signal quality of its neighboring antennas due to an increase of interference or bad user handovers.
Another example is the adjustment of the turbine blades of a wind farm where the best adjustment for one turbine may generate turbulence for its neighboring turbines and thus be suboptimal for the global wind farm objective \citep{bargiacchi2018windfarms}.

These real-life problems can be viewed as instances of a \emph{stochastic multi-agent multi-armed bandit} problem \citep{robbins1952some, bargiacchi2018windfarms} where a learner (or a \textit{central entity}) sequentially pulls a joint arm, one arm for each agent (\eg all the configuration parameters of the antennas), and receives an associated global noisy reward (\eg the signal quality over the whole network). The goal of the learner can either be to maximize the accumulated reward, implying a trade-off between exploration and exploitation, or to find the joint arm maximizing the reward, known as \emph{pure exploration} or \emph{best arm identification} \citep{bubeck2009pure,audibert2010best}. 
%The large dimension of the joint arm makes it a high-dimensional combinatorial problem (a typical wireless network of a city is made of thousands of antennas, each of them with a large number of configuration parameters). It is thus of interest to design approaches decomposing this combinatorial problem into smaller local problems based on the a priori interaction information.

In this paper we focus on the best arm identification problem in a multi-agent system for which we assume the knowledge of a coordination graph $\mathcal{G}=(V,E)$ representing the agent interactions \citep{Guestrin2002coordgraph}.

At each round $t$, a learner 
\begin{enumerate}
    \item chooses for each node $i \in V$ an arm $x_{t}^{(i)}$ in an finite arm set $\mathcal{X} \subset \mathbb{R}^d$,
    \item observes for each edge $(i,j)\in E$ a bilinear reward  $r_{t}^{(i,j)}=x_{t}^{(i)\top} {\mathbf{M}_{\star}}x_{t}^{(j)}+\eta_{t}^{(i,j)}$.
\end{enumerate}
Here, we denote by $\mathbf{M}_\star \in \mathbb{R}^{d\times d}$ the unknown parameter matrix, and $\eta^{(i,j)}_t$ a zero-mean $\sigma$-sub-Gaussian random variable for all edges $(i, j) \in E$ and round $t$. 

The goal of the central entity is to find, within a minimum number of rounds, the joint arm $(x^{(1)}_\star, \dots, x^{(|V|)}_\star)$ such that the expected global reward $\sum_{(i,j)\in E}x_\star^{(i)\top}{\mathbf{M}_{\star}}x_\star^{(j)}$ is maximized. 

%Instead of a global reward being a sum of independent linear agent rewards,  %obtained for each agents
%the global reward is now the result of the interactions between neighbouring agents.
% The reward $r^{(i,j)}_{t}$ reflects the quality of the interaction between neighboring nodes $i$ and $j$ when pulling respectively the arm $x^{(i)}_t$ and $x^{(j)}_t$ at time $t$.
% For the wireless network example, $r^{(i,j)}_{t}$ can be interpreted as the signal quality of a region according to the configuration parameters $x^{(i)}_t$ and $x^{(j)}_t$ of the neighboring antennas $i$ and $j$.
% \albert{the signal quality of the region will also depend on the other $r_t^{i,j}$. let's see if we better introduce the bilinear reward setup above but maybe we could just remove this sentence or find a better explanation such as "contribution of antenna $j$ to the signal of antenna $i$"}
%More precisely,
The reward $r^{(i,j)}_{t}$ reflects the quality of the interaction between the neighboring nodes $i$ and $j$ when pulling respectively the arm $x^{(i)}_t$ and $x^{(j)}_t$ at time $t$. For instance, when configuring handover parameters of a wireless network, $r^{(i,j)}_{t}$ can be any criterion assessing the handover quality between antenna $i$ and antenna $j$, the parameters selected by each antenna both impacting this quantity. The bilinear setting appears as a natural extension of the commonly studied linear setting to model the interaction between two agents. Furthermore, instead of a global reward being a sum of independent linear agent rewards,  %obtained for each agents
the global reward is now the result of the interactions between neighboring agents.
% as the impact on the signal quality of the region covered by the antenna $i$ according to its configuration parameters and that of its neighboring antenna $j$.
% \albert{the signal quality of the region will also depend on the other $r_t^{i,j}$. let's see if we better introduce the bilinear reward setup above but maybe we could just remove this sentence or find a better explanation such as "contribution of antenna $j$ to the signal of antenna $i$"}

As exposed in \citet{jun2019bilinear}, the bilinear reward can be written as a linear reward in a higher dimensional space:
\begin{align}
r^{(i,j)}_{t} = \vect{\left(x^{(i)}_tx^{(j)\top}_t\right)}^\top \vect{\left(\mathbf{M}_\star\right)} + \eta^{(i,j)}_t \enspace,\label{eq:reward_linear_form}
\end{align}

where for any matrix $\mathbf{A} \in \mathbb{R}^{d\times d}$, $\vect{\left(\mathbf{A}\right)}$ denotes the vector in $\mathbb{R}^{d^2}$ which is the concatenation of all the columns of $\mathbf{A}$. 

Since the unknown parameter $\mathbf{M}_{\star}$ is common to all the edges $(i,j)$ of the graph, the expected global reward % $\sum_{(i,j)}x_t^{(i)\top}{\mathbf{M}_{\star}}x_t^{(j)}$ 
at time $t$ can also be written as the scalar product $\left\langle \sum_{(i,j)\in E}\vect{\left(x^{(i)}_t x_t^{(j)\top}\right)}, \vect{(\mathbf{M}_\star)}\right\rangle$. Hence, solving the best arm identification problem in the described graphical bilinear bandit boils down to solving the same problem in a global linear bandit. Although this trick allows to use classical algorithms 
%for the problem of best arm identification 
in linear bandits, the number of joint arms is growing exponentially with the number of nodes,
making such  methods impractical.
%for the problem at stake.

Another possible way to address this problem based on equation \eqref{eq:reward_linear_form} is to consider one linear bandit per edge, with constraints between edges. For more clarity, let us define the arm set $\mathcal{Z} = \{\vect{(x x^{\prime\top})} \ | \ (x ,x^{\prime}) \in \mathcal{X}^2\}$, and let us refer to any $z \in \mathcal{Z}$ as an \emph{edge-arm} and to any $x \in \mathcal{X}$ as an \emph{node-arm}. At each round $t$, the learner chooses for each edge $(i,j)$ an edge-arm in $\mathcal{Z}$ with the constraint that for any pair of edges $(i,j)$ and $(i,k)$, if the edge-arm $\vect{(x x^{\prime\top})}$ is assigned to the edge $(i,j)$ and the edge-arm $\vect{(x^{\prime\prime} x^{\prime\prime\prime\top})}$ is assigned to the edge $(i,k)$, then it must be that  $x = x^{\prime\prime}$.

Given this constraint, how do we choose the appropriate sequence of edge-arms in order to build a good estimate of $\vect{(\mathbf{M}_\star)}$? Moreover, assuming we have built such a good estimator, is there a tractable algorithm to identify the best joint arm, or at least to find a joint arm yielding a high expected reward? In this paper, we answer these questions and provide algorithms and theoretical guarantees. 

We show that even with a perfect estimator $\vect{(\hat{\mathbf{M}})} = \vect{(\mathbf{M}_\star)}$, identifying the best joint arm is NP-Hard% (\lcf Section~\ref{sec:best_arm})
. To address this issue, we design a polynomial time twofold algorithm. Given $\vect{(\hat{\mathbf{M}})}$,  it first identifies the best edge-arm $z_\star \in \mathcal{Z}$ maximizing $\langle z_\star, \vect{(\hat{\mathbf{M}})}\rangle$. Then, it  allocates  $z_\star$ to a carefully chosen subset of edges. We show that this
%algorithm 
yields a good approximation ratio in Section~\ref{sec:best_arm}.

To build our estimator $\vect{(\hat{\mathbf{M}})}$, we rely on the G-Allocation strategy, as in \citet{SoareNIPS2014}. We show that there exists a sampling procedure over the node-arms such that the associated edge-arms follow the optimal G-allocation strategy developed in the linear bandit literature. This procedure allows us to avoid the difficulty of having to satisfy the edge-arm constraints explicitly.
%    explicitly. 
Furthermore, we analyze the sample complexity of this method. This is detailed in Section~\ref{sec:estimating_theta}.

In addition, we highlight the impact of the graph structure in Section~\ref{sec:variance} 
%on the proposed methods
and provide the explicit repercussion on the convergence rate of the algorithm for different types: star, complete, circle and matching graphs. In particular, we show that for favorable graph structures (e.g. circles), our convergence rate matches that of standard linear bandits.
Finally, Section~\ref{sec:expe} evidences the theoretical findings on numerical experiments.

\section{Related Work}
\label{sec:related_work}

% Our work is closely related to the best arm identification problem in linear bandits, as well as bilinear bandits, bandits in graphs and combinatorial and multi-agent bandits. Below, we present the various associated papers and explain how they differ from our framework and approach.

% \igor{Ca me semble interessant de dire un truc genre ``les bandits bilineaires ont recu de l'attention dernierement (reward differente [x], raffinement de la borne [y], etc.) mais ces approches ne peuvent etre utilisees dans notre cas". Un truc du genre, juste pour donner l'info de l'etat de la recherche et justifier (en 2 mots) que ca ne s'applique pas a notre probleme.}

\textbf{Best arm identification in linear bandits. } There exists a vast literature on the problem of best arm identification in linear bandits \citep{SoareNIPS2014, XuAISTATS2018, degenne2020gamification, kazerouni2019best, zaki2020explicit,jedra2020optimal}, would it be by using greedy strategies \citep{SoareNIPS2014}, rounding procedures \citep{tanner2019transductive} or random sampling \citep{TaoICML2018}. Although our problem can be formulated as a linear bandit problem, 
%as explained in the introduction 
none of the existing methods %can be straightforwardly applied if one wants a solution scaling
would scale-up with the number of agents. Nevertheless, we will be relying on classical techniques, and more specifically those developed in \citet{SoareNIPS2014}. 

\textbf{Bilinear bandits. } Bandits with bilinear rewards have been studied in \citet{jun2019bilinear}.
The authors derived a no-regret algorithm based on Optimism in the Face of Uncertainty Linear bandit (OFUL) \citep{abbasi2011improved}, using the fact that a bilinear reward can be expressed as a linear reward in higher dimension.
Our work extends their setting by considering a set of dependent bilinear bandits. Besides, the goal here is to find the best arm rather than minimizing the regret.

\textbf{Bandits and graphs. } Graphs are often used to bring structure to a bandit problem. In \citet{valko2014spectralbandits} and \citet{mannor2011bandits}, the arms are the nodes of a graph and pulling an arm gives information on the rewards of the neighboring arms. The reader can also refer to \citet{valko2020bandits} for an account on such problems. In \citet{cesa2013gang} each node is an instance of a linear bandit and the neighboring nodes are assumed to have similar unknown regression coefficients. The main difference with our setting is that the rewards of the nodes are independent.%a reward in their model is obtained for each node independently of its neighbors.

\textbf{Combinatorial and multi-agent bandits. } Allocating arms to each node of a graph to then observe a global reward is a combinatorial bandit problem \citep{cesabianchi2012}, the number of joint arms scaling exponentially with the number of nodes. This
%The combinatorial bandit problem 
has been extensively studied both in the regret-based \citep{chen2013combinatorial,Perrault2020} and the pure exploration context \citep{Chen2014,Cao2019,Jourdan2021, du2020combinatorial}. Our problem is closer to the one presented in \citet{amin2012graphical} and \citet{bargiacchi2018windfarms}, where several agents want to maximize a global team reward that can be decomposed into a sum of observable local rewards as in a \emph{semi-bandit game} \citep{audibert2011semibandits, chen2013combinatorial}. 
% In this case it is preferable to exploit the decomposition of the reward to reduce the size of the problem.
However, we study a more structured context as we assume observable bilinear rewards for each edge of the %coordination 
graph. Furthermore, note that our problem can be solved by the algorithm presented in \citet{du2020combinatorial} with a sample
complexity increasing in the number of nodes. On the contrary, we propose in this paper an algorithm with a sample complexity decreasing in the number of nodes exploiting the structure of the bilinear reward and the graph.
% In \citet{amin2012graphical}, the authors also assume that the global graph reward can be decomposed into local functions but there is no \emph{a priori} knowledge on the structure of these functions.
% On the one hand, similarly to us, the authors of \citep{bargiacchi2018windfarms} assume the knowledge of a coordination graph indicating the dependence of the agents between themselves.
% On the other hand,
Finally, most of the algorithms developed for combinatorial bandits assume the availability of an oracle to solve the combinatorial optimization problem returning the arm to play or the final best arm recommendation. We make no such assumption.
\section{Preliminaries and Notations}
\label{sec:preliminaries}
Let $\mathcal{G}=(V,E)$ be a directed graph with $V$ the set of nodes, $E$ the set of edges where we assume that
if $(i, j)  \in E$ then $(j, i) \in E$, and $\mathcal{N}(i)$ the set containing the neighbors of a node $i \in V$. We denote by $n = |V|$ the number of nodes and $m = |E|$ the number of edges. We define the \emph{graphical bilinear bandit} on the graph $\mathcal{G}$ as the setting where a learner sequentially pulls at each round $t$ a joint arm $\left(x^{(1)}_t, \dots, x^{(n)}_t\right) \in \mathcal{X}^n$, also called graph allocation or simply allocation when it is clear from the context, and then receives a bilinear reward $r^{(i,j)}_t$ for each edge $(i,j) \in E$. At each round, the joint arm can be constructed simultaneously or sequentially, however all the bilinear rewards are only revealed after the joint arm has been pulled.

We denote $K = |\mathcal{X}|$ the number of node-arms and it is assumed that %the vectors in the finite set 
$\mathcal{X}$ spans $\mathbb{R}^d$. For each round $t$ of the learning procedure and each node $i \in V$, $x^{(i)}_t \in \mathcal{X}$ represents the node-arm allocated to the node $i \in V$. For each edge $(i,j) \in E$, we denote $z^{(i,j)}_t =  \vect{(x^{(i)}_t x^{(j)\top}_t)} \in \mathcal{Z}$ the associated chosen edge-arm.

The goal is to derive an algorithm that minimizes the number of pulled joint arms required to find the one maximizing the sum of the associated expected bilinear rewards, for a given confidence level. For the sake of simplicity, we assume that the unknown parameter matrix $\mathbf{M}_\star$ in the bilinear reward is symmetric. We provide an analysis of the non-symmetric case in Appendix \ref{sec:supp_generalization}.%However, we can relax this assumption and we show in the supplementary materials that all the methods presented in this paper can be extended to the non-symmetric case. 

For any finite set $X$, $\mathcal{S}_X \triangleq \{ \lambda \in [0, 1]^{|X|} \text{, } \sum_{x \in X} \lambda_x = 1 \}$ denotes the simplex in $\mathbb{R}^{|X|}$. For any vector $x \in \bbR^d$, $\| x \|$ will denote the $\ell_2$-norm of $x$. For any square matrix $\mathbf{A} \in \bbR^{d \times d}$, we denote by $\|\mathbf{A}\|\triangleq \sup_{x: \|x\|=1} \| \mathbf{A}x\|$ the spectral norm of $\mathbf{A}$. Finally, for any vector $x \in \mathbb{R}^d$ and a symmetric positive-definite matrix $\mathbf{A} \in \mathbb{R}^{d\times d}$, we define $\|x\|_{\mathbf{A}} \triangleq \sqrt{x^\top \mathbf{A} x}$.
\section{An NP-Hard Problem}
\label{sec:best_arm}

In this section, we address the problem of finding the best joint arm given $\mathbf{M}_\star$ or a good estimator $\hat{\mathbf{M}}$. If the best edge-arm $z_\star$ is composed of a single node-arm $x_\star$, that is $z_\star = \vect{(x_\star x_\star^{\top})}$, then finding the best joint arm is trivial and the solution is to assign $x_\star$ to all nodes. Conversely, if $z_\star$ is composed of two distinct node-arms $(x_\star,x_\star^\prime)$, the problem is harder. %  For the rest of the section, we focus our analysis on the latter case.

The following theorem states that, even with the knowledge of the true parameter $\mathbf{M}_\star$, identifying the best join-arm is NP-Hard with respect to the number of nodes $n$.
\begin{theorem}
\label{theorem:nphard}
Consider a given matrix $\mathbf{M}_\star \in \mathbb{R}^{d \times d}$ and a finite arm set $\mathcal{X} \subset \mathbb{R}^d$. Unless P=NP, there is no polynomial time algorithm guaranteed to find the optimal solution of 
% \eqref{opt-glob}
\begin{align*}
\max_{\left(x^{(1)}, \ldots, x^{(n)}\right) \in\mathcal{X}^n} \sum_{(i,j)\in E} x^{(i)\top}\mathbf{M}_\star \ x^{(j)}\enspace.
\end{align*}
\end{theorem}

The proof of this theorem is in Appendix \ref{sec:supp_finding_best_arm} and relies on a reduction to the Max-Cut problem.
Hence, no matter which estimate $\hat{\mathbf{M}}$ of $\mathbf{M}_\star$ one can build,
%, even if $\hat{\mathbf{M}} =\mathbf{M}_\star$, 
the learner is not guaranteed to find in polynomial time the joint arm $\left(x^{(1)}_\star, \dots, x^{(n)}_\star\right)$ maximizing the expected global reward.
%$\sum_{(i,j)\in E} x^{(i)\top}_\star \mathbf{M}_\star x^{(j)}_\star$.
However, one can notice that, given the matrix $\mathbf{M}_\star$ or even a good enough estimate $\hat{\mathbf{M}}$, identifying the edge-arm $z^\star = \vect{(x_\star x_\star^{\prime\top})} \in \mathcal{Z}$ that maximizes the reward $z_\star^\top \vect{(\mathbf{M}_\star)}$ requires only $K^2$ reward estimations (we simply estimate all the linear reward associated to each edge-arm in $\mathcal{Z}$). Thus, instead of looking for the best joint arm explicitely, 
we will first identify the best edge-arm $z^\star$, and then allocate $z^\star$ to the largest number of edges in the graph.
%we will design a polynomial time algorithm that allocates the largest number of edge-arms $z_\star$ in the graph
We will also show that this approach gives a guarantee on its associated global reward.

Let us consider the graph allocation that places the maximum number of edge-arms $z_\star$ in $\mathcal{G}$. It is easy to show that the subgraph containing only the edges where $z_\star$ has been pulled is the largest bipartite subgraph included in $\mathcal{G}$. Recall that a graph $\mathcal{G}^\prime = (V^\prime,E^\prime)$ is a bipartite if and only if one can partition the node set $V^\prime = (V^\prime_1,V^\prime_2)$ such that
\begin{align*}
    (i,j) \in E^\prime \Rightarrow (i, j) \in V^\prime_1 \times V^\prime_2 \text{  or  } (j, i) \in V^\prime_1 \times V^\prime_2 \enspace.
\end{align*}
Notice, that if $\mathcal{G}^\prime$ is the largest bipartite subgraph in $\mathcal{G}$, the number of edges in $E^\prime$ is the maximal number of edge-arms $z_\star$ that can be allocated with a single graph allocation. 

Hence, finding the joint arm with the largest number of edge-arms $z_\star$ allocated in the graph is equivalent to finding the largest bipartite subgraph $\mathcal{G}^\prime = (V^\prime,E^\prime)$ in $\mathcal{G}$. Once that subgraph is determined, we just need to allocate to all the nodes in $V^\prime_1$ the node-arm $x_\star$ and to all the nodes of $V^\prime_2$ the node-arm $x_\star^\prime$ (which is equivalent to allocating to all the edges in $E^\prime$ the edge-arm $z_\star$).

Furthermore, we know that every $m$-edge graph contains a bipartite subgraph of at least $m/2$ edges \citep{erdos1975problems}. Therefore, we propose Algorithm~\ref{algorithm:bipartite} which iteratively constructs a bipartite subgraph and allocates the nodes accordingly to create at least $m/2$ edge-arms $z_\star$.

% looking for a bipartite subgraph of at least $m/2$ edges is a way to get closer to the best allocation and to have a guarantee on the reward obtained of being at least $1/2$ from the optimum (and thus to have $\epsilon = 1/2$). For this, 

\begin{algorithm}

    \SetKwInOut{Input}{Input}\SetKwInOut{Output}{Output}
    
    \SetAlgoLined
    
    \Input{$\mathcal{G} = (V,E)$, $\mathcal{X}$, $\mathbf{M}$}
    
    Find $(x_\star,x^\prime_\star) \in \argmax_{(x, x^\prime)\in \mathcal{X}^2} x^\top \mathbf{M} \ x^\prime$
    
    Set $V_1 = \emptyset$, $V_2 = \emptyset$
    
    \For{$i$ in $V$}{
    
        Set $n_1$ the number of neighbors of $i$ in $V_1$
        
        Set $n_2$ the number of neighbors of $i$ in $V_2$
        
        \uIf{$n_1 > n_2$}{
        
            $x^{(i)} = x^\prime_\star$
            
            $V_2 \leftarrow V_2 \cup \{i\}$
            
        }
        \Else{
        
            $x^{(i)} = x_\star$
            
            $V_1 \leftarrow V_1 \cup \{i\}$
            
        }
    }
    
    return $\mathbf{x} = \left(x^{(1)},\dots,x^{(n)}\right)$

    \caption{Bipartite graph algorithm for Best Arm Identification in Graphical Bilinear Bandits}
    \label{algorithm:bipartite}
\end{algorithm}

The following result gives the guarantee on the global reward associated to the joint arm returned by Algorithm~\ref{algorithm:bipartite}. We refer the reader to Appendix \ref{sec:supp_finding_best_arm} for the proof.

\begin{theorem}
\label{theorem:halfOptimiality}
Let us consider the graph $\mathcal{G}=(V,E)$, a finite arm set $\mathcal{X} \subset \mathbb{R}^d$ and the matrix $\mathbf{M}_\star$ given as input to Algorithm~\ref{algorithm:bipartite}. Then, the expected global reward $r = \sum_{(i,j)\in E} x^{(i)\top} \mathbf{M}_\star x^{(j)}$ associated to the returned allocation $\mathbf{x}=\left(x^{(1)},\dots,x^{(n)}\right) \in \mathcal{X}^n$ verifies:
\begin{align*}
    \frac{r - r_{\mathrm{min}}}{r_{\star} - r_{\mathrm{min}}} \geq \frac{1}{2} \enspace.
\end{align*}

where $r_{\star}$ and $r_{\mathrm{min}}$ are respectively the highest and lowest global reward one can obtain with the appropriate joint arm.
Finally, the complexity of the algorithm is in $\mathcal{O}(K^2 + n)$.
\end{theorem}

This type of approximation result is sometimes referred to as \emph{differential approximation} or \emph{$z$-approximation}, and is often viewed as a more subtle analysis than standard approximation ratio. We emphasize that finding a better ratio than $\frac{1}{2}$ is a very hard task: such a finding %since the Max-Cut problem is a reduction of our problem, a differential approximation ratio better than $1/2$
would immediately yield an improved differential approximation ratio for the Max-Cut problem, which is an opened problem since 2001 \citep{hassin2001z}.

% Our $1/2$-approximation bound is closely related to the ones used to solve the NP-Hard problem of Max-Cut . However, the difference between these types of approximation and ours is that we have related the reward to $r_\star$ and $r_{\mathrm{min}}$ while the approximation bound for the Max-Cut problem is related to the actual value of the Max-Cut, which can be much less than $r_\star$. Therefore, when \citet{goemans1995improved} gives a 0.87856 polynomial-time approximation upper bound under the Unique Game Conjecture (which is, to the best of our knowledge, a state-of-the-art bound), it is relative to the value $\mathrm{MC}$ of the Max-Cut, which makes the guarantee very graph-dependent since 0.87856 times $\mathrm{MC}$ could for some graph be strictly less than $\frac{1}{2} (r_{\star} - r_{\mathrm{min}}) + r_{\mathrm{min}}$. Whereas our $1/2$-approximation relative to $r_\star$ and $r_{\mathrm{min}}$ is valid for any graph. 

% \albert{refs? ou c'est celles qui sont donnees plus bas dans le paragraphe}

\section{Construction of the Estimate $\hat{\mathbf{M}}$}
\label{sec:estimating_theta}

In the previous section, we designed a polynomial time method that computes a $1/2$-approximation to the NP-Hard problem of finding the best joint arm given $\mathbf{M}_\star$. Notice that $\mathbf{M}_\star$ is only used to identify the best edge-arm $z_\star$. Thus, using an estimate $\hat{\mathbf{M}}$  of $\mathbf{M}_\star$ having the following property: %in Algorithm~\ref{algorithm:bipartite}, such that 
\begin{align}
\label{eq:same_arg_max}
    \argmax_{z \in \mathcal{Z}} z^\top \vect{(\hat{\mathbf{M}})} = \argmax_{z \in \mathcal{Z}} z^\top \vect{(\mathbf{M}_\star)} \enspace,
\end{align} 
would still allow us to identify $z_\star$, and would thus give us the same guarantees.

In this section we tackle the problem of pulling the edge-arms during the learning procedure such that the estimated unknown parameter verifies \eqref{eq:same_arg_max} in as few iterations as possible. To do so, we first formalize the objective related to the %linear bandit problem where the existing methods cannot be directly used because of the dependencies between the edges.
linearized version of the problem. Then, we propose an algorithm reaching the given objective with high probability while satisfying the edge-arms constraints.

We denote by $\theta_\star = \vect{\left(\mathbf{M}_\star\right)}$ the parameter of the linearized problem and $\hat{\theta}_t$ the Ordinary Least Squares (OLS) estimate of $\theta_\star$ computed with all the data collected up to round $t$. The empirical gap between two edge-arms $z$ and $z^\prime$ in $\mathcal{Z}$ is denoted $\hat{\Delta}_t\left(z,z^\prime\right) \triangleq \left(z - z^\prime\right)^{\top} \hat{\theta}_t$.

\subsection{A Constrained G-Allocation}

The goal here is to define the optimal sequence $\left(z_1,\dots, z_{mt}\right) \in \mathcal{Z}^{mt}$ that should be pulled in the first $t$ rounds so that \eqref{eq:same_arg_max} is reached as soon as possible. A natural approach is to rely on classical strategies developed for best arm identification in linear bandits.
% For the time being, let us put aside the constraints imposed by the graph and assume that the $m$ edge-arms in $\mathcal{Z}$ can be picked independently for each round $t$.
Most of the known strategies (see \eg \citet{SoareNIPS2014,XuAISTATS2018,tanner2019transductive}) are based on a bound of the gap error $\vert (\theta_{\star} - \hat{\theta}_t)^{\top}(z - z^\prime) \vert$ for all $z, z^\prime \in \mathcal{Z}$. %This allows to derive a stopping condition, indicating when one can stop pulling edge-arms because the OLS estimate $\hat \theta$ is good enough to discriminate the edge-arms.
This bound is then used to derive a stopping condition, indicating a sufficient number of rounds $t$ after which the OLS estimate $\hat{\theta}_t$ is precise enough to ensure the identification of the best edge-arm, with high probability. 

Let $\delta \in (0,1)$ and let $\mathbf{A}_t = \sum_{i=1}^{mt} z_iz_i^\top$ be the matrix computed with the $mt$ edge-arms constructed during $t$ rounds. Following the steps of \citet{SoareNIPS2014}, we can show that if there exists $z \in \mathcal{Z}$ such that for all $z^\prime \in \mathcal{Z}$ the following holds:
%$\exists z \in \mathcal{Z}, \forall z^\prime \in \mathcal{Z},$
\begin{align}
    \|z - z^\prime \|_{\mathbf{A}_t^{-1}} \sqrt{8 \sigma^2 \log{\left(\frac{6m^2t^2 K^4}{\delta \pi^2} \right)}} \leq \hat{\Delta}_t\left(z,z^\prime\right) \enspace,
    \label{eq:stopping_condition}
\end{align}
then with probability at least $1-\delta$, the OLS estimate $\hat{\theta}_t$ leads to the best edge-arm. Details of the derivation are given in Appendix \ref{sec:supp_stopping_condition}.
As mentioned in \citet{SoareNIPS2014}, by noticing that $\max_{(z,z^{\prime}) \in \mathcal{Z}^2} \left\|z - z^\prime\right\|_{\mathbf{A}_t^{-1}} \leq 2 \max_{z^ \in \mathcal{Z}} \left\|z \right\|_{\mathbf{A}_t^{-1}}$, an admissible strategy is to pull edge-arms minimizing $\max_{z \in \mathcal{Z}} \left\|z \right\|_{\mathbf{A}_t^{-1}}$ in order to satisfy the stopping condition as soon as possible. More formally, one wants to find the sequence of edge-arms $\mathbf{z}_{mt}^\star = \left(z_1^\star, \dots, z_{mt}^\star\right)$ such that:
\begin{align}
    \label{eq:G_allocation_strategy}
    \mathbf{z}_{mt}^{\star} \in \argmin_{\left(z_1, \dots, z_{mt}\right)} \max_{z^\prime \in \mathcal{Z}} \; {z^\prime}^\top \left(\sum_{i=1}^{mt}z_iz_i^\top\right)^{-1}{z^\prime} \enspace.
    \tag{G-opt-$\mathcal{Z}$}
\end{align}
This is known as \emph{G-allocation} (see \eg \citet{Pukelsheim2006, SoareNIPS2014}) and is NP-hard to compute \citep{CivrilNPHard,WelchNPHard}. One way to find an approximate solution is to rely on a convex relaxation of the optimization problem \eqref{eq:G_allocation_strategy} and first compute a real-valued allocation $\lambda^\star \in \mathcal{S}_{\mathcal{Z}}$ such that
\begin{align}
\label{eq:relaxed_G_allocation_Z}
   \lambda^\star \in \argmin_{\lambda \in \mathcal{S}_{\mathcal{Z}}}\max_{z^\prime\in \mathcal{Z}} {z^\prime}^\top\left(\sum_{z \in \mathcal{Z}} \lambda_z z z^\top\right)^{-1}z^\prime \enspace. 
   \tag{G-relaxed-$\mathcal{Z}$}
\end{align} 
One could either use random sampling to draw edge-arms as i.i.d.\ samples from the $\lambda^\star$ distribution or rounding procedures to efficiently convert each $\lambda^\star_z$ into an integer. However, these methods do not take into account the graphical structure of the problem, and at a given round, the $m$ chosen edge-arms may result in two different assignments for the same node. Therefore, random sampling or rounding procedures cannot be straightforwardly used to select edge-arms in $\mathcal{Z}$.
Nevertheless, \eqref{eq:relaxed_G_allocation_Z} still gives a valuable information on the number of times, in proportion, each edge-arm $z \in \mathcal{Z}$ must be allocated to the graph. In the next section, we present an algorithm satisfying both the proportion requirements and the graphical constraints.

\subsection{Random Allocation over the Nodes}
\label{sec:mu_to_lambda}
% In this section, we design an algorithm that selects $mt$ edge-arms in $\mathcal{Z}$ that verify \eqref{eq:G_allocation_strategy} while satisfying the graphical constraints. To do so, we base our approach on a randomized method and show that one can draw node-arms in $\mathcal{X}$ and allocate them to the graph such that the associated edge-arms follow the probability distribution $\lambda^\star$ solution of \eqref{eq:relaxed_G_allocation_Z}. By directly allocating node-arms to the nodes, we avoid the difficult task of choosing edge-arms and trying to allocate them to the graph while ensuring that every node has an unique assignment.

Our algorithm is based on a randomized method directly allocating node-arms to the nodes and thus avoiding the difficult task of choosing edge-arms and trying to allocate them to the graph while ensuring that every node has an unique assignment. The validity of this random allocation is based on Theorem \ref{theorem:mu_to_lambda} below showing that one can draw node-arms in $\mathcal{X}$ and allocate them to the graph such that the associated edge-arms follow the probability distribution $\lambda^\star$ solution of \eqref{eq:relaxed_G_allocation_Z}.

\begin{theorem}
\label{theorem:mu_to_lambda}
Let $\mu^\star$ be a solution of the following optimization problem: 
\begin{align}
    \min_{\mu \in \mathcal{S}_{\mathcal{X}}} \max_{x^\prime\in \mathcal{X}} {x^\prime}^\top\left(\sum_{x \in \mathcal{X}} \mu_x x x^\top\right)^{-1}x^\prime \label{eq:relaxed_G_allocation_X} \enspace. \tag{G-relaxed-$\mathcal{X}$}
\end{align}
Let $\lambda^\star \in \mathcal{S}_{\mathcal{Z}}$ be defined for all $z = \vect{\left(xx^{\prime\top}\right)} \in \mathcal{Z}$ by $\lambda_z^\star = \mu^\star_x \mu^\star_{x^\prime}$. Then, $\lambda^\star$ is a solution of \eqref{eq:relaxed_G_allocation_Z}.
%Then, there exists $\lambda^\star$ solution of \eqref{eq:relaxed_G_allocation_Z} such that for all $z \in \mathcal{Z}$, where $z = \vect{\left(xx^{\prime\top}\right)}$ with $(x,x^\prime) \in \mathcal{X}^2$, $\lambda_z^\star = \mu_x^\star \times \mu_{x^\prime}^\star$.
\end{theorem}

% Assume that we have a graphical bilinear bandit on a graph $\mathcal{G}=(V,E)$ and let $\mu^\star$ be a solution of the following optimization problem:
% \begin{align}
%     \min_{\mu \in \mathcal{S}_K} \max_{x^\prime\in \mathcal{X}} {x^\prime}^\top\left(\sum_{x \in \mathcal{X}} \mu_x x x^\top\right)^{-1}x^\prime \label{eq:relaxed_G_allocation_X} \enspace.
% \end{align}
% For all round $t \geq 0$, and all nodes $i \in V$, if $x_t^{(i)}$ is drawn from $\mu^{\star}$, then for all couple of neighbors $(i,j) \in E$ the probability distribution of the constructed arms $z^{(i,j)}_t$ is $\lambda^\star$, a solution of  
% \begin{align}
%     \min_{\lambda \in \mathcal{S}_{K^2}} \max_{z^\prime\in \mathcal{Z}} {z^\prime}^\top\left(\sum_{z \in \mathcal{Z}} \lambda_z zz^\top\right)^{-1}z^\prime \label{eq:relaxed_G_allocation_Z} \enspace.
% \end{align}

\paragraph{Sketch of proof.} The objective at the optimum in \eqref{eq:relaxed_G_allocation_X} and \eqref{eq:relaxed_G_allocation_Z} are respectively equal to $d$ and $d^2$ which is the dimension of their respective problem, a result known as the Equivalence Theorem \citep{KieferEquivalenceThm1960}. Thus, by multiplying the optimum value of $\eqref{eq:relaxed_G_allocation_X}$ by itself, we can show that for all $z\in \mathcal{Z}$ where $z = \vect{\left(xx^{\prime\top}\right)}$ with $(x, x^\prime) \in \mathcal{X}^2$, $\lambda^\star_z$ can be written as the product $\mu_x^\star \mu_{x^\prime}^{\star}$. We refer 
%the reader 
to the Appendix \ref{sec:supp_parameter_estimation} for the detailed proof.

This theorem implies that, at each round $t > 0$ and each node $i \in V$, if $x_t^{(i)}$ is drawn from $\mu^{\star}$, then for all pairs of neighbors $(i,j) \in E$ the probability distribution of the associated edge-arms $z^{(i,j)}_t$ follows $\lambda^\star$. Moreover, as $\mu^\star$ is a distribution over the node-arm set $\mathcal{X}$, $\lambda^\star$ is a joint (product) probability distribution on $\mathcal{X}^2$ with marginal $\mu^\star$.

We apply the Frank-Wolfe algorithm \citep{frank1956algorithm} to compute the solution $\mu^\star$ of \eqref{eq:relaxed_G_allocation_X}, as it is more suited to optimization tasks on the simplex than projected gradient descent. Although we face a min-max optimization problem, we notice that the function $h(\mu) = \max_{x^\prime\in \mathcal{X}} x^{\prime\top}\left(\sum_{x \in \mathcal{X}} \mu_x xx^\top \right)^{-1}x^\prime$ is convex. We refer the reader to Appendix~\ref{sec:supp_experiments} and references therein for a proof on the convexity of $h$ and a discussion about using Frank-Wolfe for solving \eqref{eq:relaxed_G_allocation_X}.

Given the characterization in Theorem~\ref{theorem:mu_to_lambda} and our objective to verify the stopping condition in \eqref{eq:stopping_condition}, we present our sampling procedure in Algorithm~\ref{algorithm:Randomized_G_allocation_for_GBB}. We also note that at each round the sampling of the node-arms can be done in parallel.
\begin{algorithm}[h]
    \SetKwInOut{Input}{Input}\SetKwInOut{Output}{Output}
    \SetAlgoLined
    \Input{graph $\mathcal{G} = (V,E)$, arm set $\mathcal{X}$}
    Set $A_0 = I$ ; $b_0 = 0$ ; $t = 1$;
    
    Apply the Frank-Wolfe algorithm to find $\mu^\star$ solution of~\eqref{eq:relaxed_G_allocation_X}.
    
    \While{stopping condition \eqref{eq:stopping_condition} is not verified}{
        \color{blue} // Sampling the node-arms \color{black}
        
        Draw $x_t^{(1)}, \ldots, x_t^{(n)} \stackrel{\mathrm{iid}}{\sim} \mu^{\star}$ and obtain for all $(i,j)$ in $E$ the rewards $r_t^{(i,j)}$; 
        
        \color{blue} // Estimating $\hat{\theta}_t$ with the associated edge-arms \color{black}
        
        $\mathbf{A}_t = \mathbf{A}_{t-1} + \sum_{(i,j) \in E} z_{t}^{(i,j)}z_{t}^{(i,j)\top}$;
        
        $b_t = b_{t-1} + \sum_{(i,j)\in E} z_{t}^{(i,j)} r_t^{(i,j)}$;
        
        $\hat{\theta}_t = \mathbf{A}_t^{-1}b_t$ 
        
        $t \leftarrow t + 1$;
    }
    return $\hat{\theta}_t$
    
    \caption{Randomized G-Allocation strategy for Graphical Bilinear Bandits}
    \label{algorithm:Randomized_G_allocation_for_GBB}
\end{algorithm}

This sampling procedure implies that each edge-arm follows the optimal distribution $\lambda^\star$. However, if we take the number of times each $z \in \mathcal{Z}$ appears in the $m$ pulled edge-arms of a given round, we might notice that the observed proportion is not close to $\lambda^\star_z$, regardless of the size of $m$. This is due to the fact that the $m$ edge-arms are not independent because of the graph structure (\lcf Section~\ref{sec:variance}). Conversely, since each group of $m$ edge-arms are independent from one round to another, the proportion of each $z \in \mathcal{Z}$ observed among the $mt$ pulled edge-arms throughout $t$ rounds is close to $\lambda^\star_z$.

One may wonder if deterministic rounding procedures could be used instead of random sampling on $\mu^\star$, as it is done in many standard linear bandit algorithms \citep{SoareNIPS2014, tanner2019transductive}. Applying rounding procedure on $\mu^\star$ gives the number of times each node-arm $x \in \mathcal{X}$ should be allocated to the graph. However, it does not provide the actual allocations that the learner must choose over the $t$ rounds to optimally pull the associated edge-arms (\emph{i.e.,}~pull edge-arms following $\lambda^\star$). Thus, although rounding procedures give a more precise number of times each node-arm should be pulled, the problem of allocating them to the graph remains open, whereas by concentration of the measure, randomized sampling methods imply that the associated edge-arms follow the optimal probability distribution $\lambda^\star$.
% Hence, our theorem~\ref{theorem:mu_to_lambda} simply provides a way of optimally constructing the edge-arms that follow $\lambda^\star$ from the probability distribution $\mu^\star$ through randomness. 
In this paper, we present a simple and standard randomized G-allocation strategy, but other more elaborated methods could be considered, as long as they include the necessary randomness.

% One way to solve this hard task is to randomly pull  the allocation is done randomly over the nodes and we pull during the $t$ rounds each node-arm $x \in \mathcal{X}$ as many times as the rounding procedure tells us to pull, the empirical probability distribution of the constructed edge-arms might converge to $\lambda^\star$ as $t$ tends to $\infty$.
% However, one way to do this solve this difficult problem using randomization is to create the set containing each node-arm in $\mathcal{X}$ as many time as the rounding procedure tells us to pull and to randomly sample the node-arms present in this deterministic sequence given by the rounding procedure and allocate them to each node of the graph. 

\textbf{On the choice of the G-allocation problem. } We have considered the G-allocation optimization problem \eqref{eq:G_allocation_strategy}, however, one could want to directly minimize $\max_{(z,z^{\prime}) \in \mathcal{Z}^2} \left\|z - z^\prime\right\|_{\mathbf{A}_t^{-1}}$, known as the XY-allocation \citep{SoareNIPS2014, tanner2019transductive}. 
Hence, one may want to construct edge-arms that follow the distribution $\lambda^{\star}_{\mathrm{XY}}$ solution of the relaxed XY-allocation problem:
\begin{align*}
    \min_{\lambda} \max_{z^\prime,z^{\prime\prime}} \; \left(z^\prime - z^{\prime\prime}\right)^\top\left(\sum_{z\in \mathcal{Z}} \lambda_{z} zz^\top\right)^{-1}\left(z^\prime - z^{\prime\prime}\right) \enspace.
\end{align*}
Although efficient in the linear case, this approach outputs a distribution %The reason we do not follow this alternative approach is that the obtained 
$\lambda^{\star}_{\mathrm{XY}}$ which is not a joint probability distribution of two independent random variables, and so cannot be decomposed as the product of its marginals. Hence, there is no algorithm that allocates \emph{identically} and \emph{independently} the nodes of the graph to create edge-arms following $\lambda^\star_{\mathrm{XY}}$. Thus, we will rather deal with the upper bound given by the G-allocation as it allows sampling over the nodes.

\textbf{Static design versus adaptive design.}
Adaptive designs as proposed for example in \citet{SoareNIPS2014} and \citet{tanner2019transductive} provide a strong improvement over static designs in the case of linear
bandits. In our particular setting however, it is crucial to be able to adapt the edge-arms sampling rule to the node-arms, which is possible thanks to Theorem~\ref{theorem:mu_to_lambda}. This result requires a set of edge-arms $\mathcal{Z}$ expressed as a product of node-arms set $\mathcal{X}$. Extending the adaptive design of \citet{tanner2019transductive} to our setting would eliminate edge-arms from $\mathcal{Z}$ at each phase, without trivial guarantees that the newly obtained edge-arms set $\mathcal{Z}^\prime \subset \mathcal{Z}$ could still be derived from another node-arms set $\mathcal{X}^\prime \subset \mathcal{X}$. An adaptive approach is definitely a natural and promising extension of our method, and is left for future work.

% We now prove the validity of this sampling procedure by controlling the relative error of the approximation $\max_{z \in \mathcal{Z}} z^{ \top} \mathbf{A}_t^{-1} z$, which we obtain after $t$ rounds and the sampled $mt$ edges-arms forming the matrix $\mathbf{A}_t$, with respect to the optimum of the relaxed G-allocation optimization problem $\max_{z^\prime\in \mathcal{Z}} {z^\prime}^\top\left(mt\sum_{z \in \mathcal{Z}} \lambda_z^\star z z^\top\right)^{-1}z^\prime$. 

\subsection{Convergence Analysis}
\label{sec:convergence-analysis}

We now prove the validity of the random sampling procedure detailed in Algorithm~\ref{algorithm:Randomized_G_allocation_for_GBB} by controlling the quality of the approximation $\max_{z \in \mathcal{Z}} z^{ \top} \mathbf{A}_t^{-1} z$ with respect to the optimum of the G-allocation optimization problem $\max_{z^\prime\in \mathcal{Z}} {z^\prime}^\top\left(\sum_{i = 1}^{mt} z_i^\star z_i^{\star\top}\right)^{-1}z^\prime$ described in \eqref{eq:G_allocation_strategy}. As is usually done in the optimal design literature (see \eg \citet{Pukelsheim2006,SoareNIPS2014,SagnolPhd2010}) we bound the relative error $\alpha$:
\[
\max_{z \in \mathcal{Z}} z^{ \top} \mathbf{A}_t^{-1} z \leq (1+\alpha)\max_{z^\prime\in \mathcal{Z}} {z^\prime}^\top\left(\sum_{i = 1}^{mt} z_i^\star z_i^{\star\top}\right)^{-1}z^\prime \enspace.
\]
% We will show how concentration inequalities on random matrices can be used to provide theoretical bound on the convergence rate of randomized sampling.
Our analysis relies on several results from matrix concentration theory. One may refer for instance to \citet{tropp2015introduction} and references therein for an extended introduction on that matter. We first introduce a few additional notations.

Let $f_{\mathcal{Z}}$ be the function such that for any non-singular matrix $\mathbf{Q} \in \mathbb{R}^{d^2 \times d^2}$, $f_{\mathcal{Z}}(\mathbf{Q}) = \max_{z\in \mathcal{Z}} z^\top \mathbf{Q}^{-1} z$ and for any distribution $\lambda \in \mathcal{S}_{\mathcal{Z}}$ let $\Sigma_{\mathcal{Z}}(\lambda) \triangleq  \sum_{z \in \mathcal{Z}} \lambda_z zz^{\top}$ be the associated covariance matrix. Finally let $\mathbf{A}^{\star}_t = \sum_{i=1}^{mt} z_i^{\star}z_i^{\star\top}$ be the G-optimal design matrix constructed during $t$ rounds.

For $i \in \{1, \ldots,  n\}$ and $s \in \{1, \ldots, t \}$, let $X^{(i)}_{s}$ be \iid random vectors in $\mathcal{X}$ such that for all $x \in \mathcal{X}$, 
\[
    \mathbb{P}\left(X_1^{(1)} = x\right) = \mu^\star_x \enspace.
\]
% $\mathbb{P}(X_1^{(1)} = x) = \mu^\star_x$.
Each $X^{(i)}_{s}$ is to be viewed as the random arm pulled at round $s$ for the node $i$. 
% We define $\mathbf{Z}_1, \dots, \mathbf{Z}_t$ as the \iid random matrices in $\mathbb{R}^{d^2 \times d^2}$ such that for all $s \in \{1,\dots, t \}$, $\mathbf{Z}_{s} = \sum_{(i,j)\in E} \vect{\left(X^{i}_{s}X^{j\top}_{s}\right)} \vect{\left(X^{i}_{s}X^{j\top}_{s}\right)}^\top$. 
Using this notation, the random design matrix $\mathbf{A}_t$ can be defined as 
\begin{align*}
    \mathbf{A}_t = \sum_{s=1}^{t} \sum_{(i,j)\in E}  \vect{\left(X^{(i)}_{s}X^{(j)\top}_{s}\right)} \vect{\left(X^{(i)}_{s}X^{(j)\top}_{s}\right)}^\top \enspace.
    % \label{eq:random_A_t}
\end{align*}
One can first observe that $f_{\mathcal{Z}}(\mathbf{A}_t)$ can be bounded by the following quantity:
\begin{align*}
    f_{\mathcal{Z}}(\mathbf{A}_t)
    &=   \max_{z \in \mathcal{Z}} \;z^\top \left(\mathbf{A}_t^{-1} - \left(\mathbb{E}\mathbf{A}_t\right)^{-1} + \left(\mathbb{E}\mathbf{A}_t\right)^{-1}\right) z\\
    &\leq \max_{z \in \mathcal{Z}}\; z^\top \left(\mathbf{A}_t^{-1} - \left(\mathbb{E}\mathbf{A}_t\right)^{-1}\right) z \\
    &\phantom{\leq}+ f_{\mathcal{Z}}\left(mt \Sigma_{\mathcal{Z}}(\lambda^\star)\right) \\
    &\leq \max_{z \in \mathcal{Z}}\;\|z\|^2 \|\mathbf{A}_t^{-1} - \left(\mathbb{E}\mathbf{A}_t\right)^{-1}\| \\
    &\phantom{\leq}+ f_{\mathcal{Z}}\left(\mathbf{A}^{\star}_t\right).
\end{align*}
Hence, one needs a bound on the maximum eigenvalue of $\mathbf{A}_t^{-1} - (\mathbb{E} \mathbf{A}_t)^{-1}$. Simple linear algebra leads to:
\[
  \mathbf{A}_t^{-1} - (\mathbb{E} \mathbf{A}_t)^{-1}
  = \mathbf{A}_t^{-1}( \mathbb{E}\mathbf{A}_t - \mathbf{A}_t)(\mathbb{E} \mathbf{A}_t)^{-1}.
\]
Thus, in addition to bounding the maximum eigenvalue of $\mathbf{A}_t^{-1}$, which is equal to the minimum eigenvalue of $\mathbf{A}_t$, we need a bound on $\|\mathbf{A}_t - \bbE\mathbf{A}_t\|$. %It may be derived from the following concentration results.
It may be derived from concentration results on sum of random matrices derived in \citet{tropp2015introduction}.
% \begin{prop}[\citet{tropp2015introduction}, Chapter 5]
%   \label{prop:chernoff-matrix}
%   Let $\mathbf{A}_t$ be the random matrix defined in \eqref{eq:random_A_t}, such that there exists $L > 0$ verifying $\mathbf{0} \preceq \mathbf{A}_1 \preceq mL \mathbf{I}$. Then, for any $0 < \varepsilon < 1$, one can lowerbound $\lambda_{\mathrm{min}}(\mathbf{A}_t)$ as follows:
%   \[
%     \bbP(\lambda_{\mathrm{min}}(\mathbf{A}_t) \leq (1 - \varepsilon) \lambda_{\mathrm{min}}(\bbE \mathbf{A}_t)) \leq
%     d^2 e^{-\displaystyle\frac{t \varepsilon^2 \lambda_{\mathrm{min}}(\bbE \mathbf{A}_1) }{2mL}}.
%   \]
%   If in addition, there exists some $v > 0$, such that ${\| \bbE\big[(\mathbf{A}_1 - \bbE\mathbf{A}_1)^2 \big] \| \leq v}$, then for any $u > 0$, one has
%   \[
%     \bbP \left( \left\| \mathbf{A}_t - \bbE\mathbf{A}_t \right\| \geq \sqrt{2tvu} +\frac{mLu}{3} \right) \leq d^2 e^{-u}\enspace.
%   \]
% \end{prop}
% The proofs of these results may be found in~\citep{tropp2015introduction}. 
% We are now ready to state the result controlling the relative error obtained with our randomized sampling allocation. The proof can be found in the Appendix \ref{sec:supp_parameter_estimation}.
We now state the result controlling the relative error obtained with our randomized sampling allocation. The proof can be found in the Appendix \ref{sec:supp_parameter_estimation}.

\begin{theorem}
    \label{theorem:convergence}
    Let $\lambda^\star$ be a solution of the optimization problem~\eqref{eq:relaxed_G_allocation_Z}. Let $0 \leq \delta \leq 1$ and let $t_0$ be such that
    \[
    % t_0 = 2L \big\|\Sigma_{\mathcal{Z}}(\lambda^{\star})^{-1}\big\| \log(2d^2/\delta),
    t_0 = 2L d^2 \log(2d^2/\delta)/\nu_{\mathrm{min}},
    \]
    where $L = \max_{z \in \mathcal{Z}} \|z\|^2$ and $\nu_{\mathrm{min}}$ is the smallest eigenvalue of the covariance matrix$\frac{1}{K^{2}}\sum_{z\in\mathcal{Z}}zz^{\top}$.
    Then, at each round $t \geq t_0$ with probability at least $1 - \delta$, the randomized G-allocation strategy for graphical bilinear bandit in Algorithm~\ref{algorithm:Randomized_G_allocation_for_GBB} produces a matrix $\mathbf{A}_t$ such that:
    \begin{align*}
        f_{\mathcal{Z}}(\mathbf{A}_t) \leq (1 + \alpha) f_{\mathcal{Z}}(\mathbf{A}^{\star}_t)
    \end{align*}
    where 
    \begin{align*}
        \hspace{-.5em}\alpha = \frac{L d^2}{m\nu_{\mathrm{min}}^2} \sqrt{\frac{2v}{t}\log\left(\frac{2d^2}{\delta}\right)} + o\left(\frac{1}{\sqrt{t}}\right),
    \end{align*}
    and $v \triangleq \bbE\big[(\mathbf{A}_1 - \bbE\mathbf{A}_1)^2 \big]$.
\end{theorem}

% We refer to the supplementary materials for the proof.

% Given this theorem, we are now able to state the convergence rate of our solution to the solution of the optimum of the G-allocation optimization problem formulated in \eqref{eq:G_allocation_strategy}. In fact, since the objective of the relaxed problem~\eqref{eq:relaxed_G_allocation_strategy} is lower than the original objective~\eqref{eq:G_allocation_strategy}% $f_{\mathcal{Z}}(mt \times \Sigma_{\mathcal{Z}}(\lambda^{\star})) \leq \min_{\left(z_1, \dots, z_{mt}\right)}\max_{z \in \mathcal{Z}} z^\top \left(\sum_{i=1}^{mt}z_iz_i^\top\right)^{-1}z$
% , we have
% \begin{align*}
%     f_{\mathcal{Z}}(\mathbf{A}_t) &\leq (1 + \alpha) \min_{\left(z_1, \dots, z_{mt}\right)}\max_{z \in \mathcal{Z}} z^\top \left(\sum_{i=1}^{mt}z_i z_i^\top\right)^{-1}z \enspace.
% \end{align*}
We have just shown that the approximation value $\max_{z \in \mathcal{Z}} z^{ \top} \mathbf{A}_t^{-1} z$ converges to the optimal value with a rate of $O\left(\sqrt{v}/(m\sqrt{t})\right)$. In Section~\ref{sec:variance}, we show that the best case graph implies a $v = O\left(m\right)$ matching the convergence rate $O\left(1/\sqrt{mt}\right)$ of a linear bandit algorithm using randomized sampling to pull $mt$ edge-arms without (graphical) constraints. Moreover, we will see that the worst case graph implies that $v = O\left(m^2\right)$.

Since we filled the gap between our constraint objective and the problem of best arm identification in linear bandits, thanks to Theorem~\ref{theorem:mu_to_lambda} and \ref{theorem:convergence}, we are able to extend known results for best arm identification in linear bandits on the sample complexity and its associated lower bound.

\begin{corollary}[\citet{SoareNIPS2014}, Theorem 1]
If the G-allocation is implemented with the random strategy of Algorithm~\ref{algorithm:Randomized_G_allocation_for_GBB}, resulting in an $\alpha$-approximation, then with probability at least $1 - \delta$, the best arm obtained with $\hat{\theta}_t$ is $z_\star$ and 
\begin{align}
    t \leq \frac{128\sigma^2 d^2 (1+ \alpha)\log{\left(\frac{6m^2t^2 K^4}{\delta \pi} \right)} }{m\Delta_{\mathrm{min}}^2} \enspace, \notag
\end{align}
where $\Delta_{\mathrm{min}} = \min_{z \in \mathcal{Z} \setminus \{z_\star\}} (z_\star - z)^\top \theta_\star$.
\end{corollary}

Moreover, let $\tau$ be the number of rounds sufficient for any algorithm to determine the best arm with probability at least $1-\delta$. A lower bound on the expectation of $\tau$ can be obtained from the one derived for the problem of best arm identification in linear bandits (see \eg Theorem 1 in \citet{tanner2019transductive}):
\vspace{-0.15cm}
\[
    %\mathbb{E}[\tau \ | \ \theta_\star]
    \mathbb{E}[\tau] 
    \geq \min_{\lambda \in \mathcal{S}_{\mathcal{Z}}} \max_{z \in \mathcal{Z}\setminus \{z_\star\}} \log\left(\frac{1}{2.4\delta}\right) \frac{2\sigma^2 \|z_\star - z\|^2_{\Sigma_{\mathcal{Z}}(\lambda)^{-1}}}{m\left(\left(z_\star - z\right)^{\top}\theta_\star\right)^2} \enspace.
\]
As observed in \citet{SoareNIPS2014} this lower bound can be upper bounded, in the worst case, by $4\sigma^2d^2/(m\Delta_{\mathrm{min}}^2)$ which matches our bound up to log terms and the relative error $\alpha$.

\section{Influence of the Graph Structure on $v$}
\label{sec:variance}
The convergence bound in Theorem~\ref{theorem:convergence} depends on $v = \bbE\big[(\mathbf{A}_1 - \bbE\mathbf{A}_1)^2 \big]$. In this section, we characterize the impact of the graph structure on this quantity and, by extension, on the convergence rate. First of all, recall that 
\begin{align*} 
\mathbf{A}_{1} = \sum_{(i,j)\in E} \vect{\left(X^{(i)}_{1}X^{(j)\top}_{1}\right)} \vect{\left(X^{(i)}_{1}X^{(j)\top}_{1}\right)}^\top \enspace.
\end{align*}
Let denote $\mathbf{A}_{1}^{(i,j)} = \vect{(X^{(i)}_{1}X^{(j)\top}_{1})} \vect{(X^{(i)}_{1}X^{(j)\top}_{1})}^\top$ such that $\mathbf{A}_{1} = \sum_{(i,j)\in E} \mathbf{A}_{1}^{(i,j)}$ and let define for any random matrices $\mathbf{A}$ and $\mathbf{B}$ the operators $\var(\mathbf{A})\triangleq {\bbE\big[(\mathbf{A} - \bbE[\mathbf{A}])^2 \big]}$ and $\cov(\mathbf{A}, \mathbf{B})\triangleq {\bbE\big[(\mathbf{A} - \bbE [\mathbf{A}])(\mathbf{B} - \bbE[\mathbf{B}]) \big]}$. We can derive the variance of $\mathbf{A}_1$ as follows: 
\begin{align}
   \var\left(\mathbf{A}_1\right) =& \sum_{(i,j)\in E} \var\left(\mathbf{A}^{(i,j)}_1\right) \notag\\ 
   &+ \sum_{(i,j)\in E} \sum_{\substack{(k,l)\in E\\ (k,l) \neq (i,j)}}\cov(\mathbf{A}^{(i,j)}_1, \mathbf{A}^{(k,l)}_1) \notag .
\end{align}
One can decompose the sum of the covariances into three groups: a first group where $k \neq i,j$ and $l \neq i,j$ which means that the two edges do not share any node and $\cov(\mathbf{A}^{(i,j)}_1, \mathbf{A}^{(k,l)}_1) = \mathbf{0}$, and two other groups where the edges share at least one node. For all edges $(i,j) \in E$ we consider either the edges $(i,k) \in E$ where $k \neq j$, yielding $\cov(\mathbf{A}^{(i,j)}_1, \mathbf{A}^{(i,k)}_1)$ or the edges $(j,k) \in E$, yielding $\cov(\mathbf{A}^{(i,j)}_1, \mathbf{A}^{(j,k)}_1)$.
\begin{figure*}

\centering
\subfloat{\label{fig:d}
\centering
\includegraphics[width=0.47\linewidth]{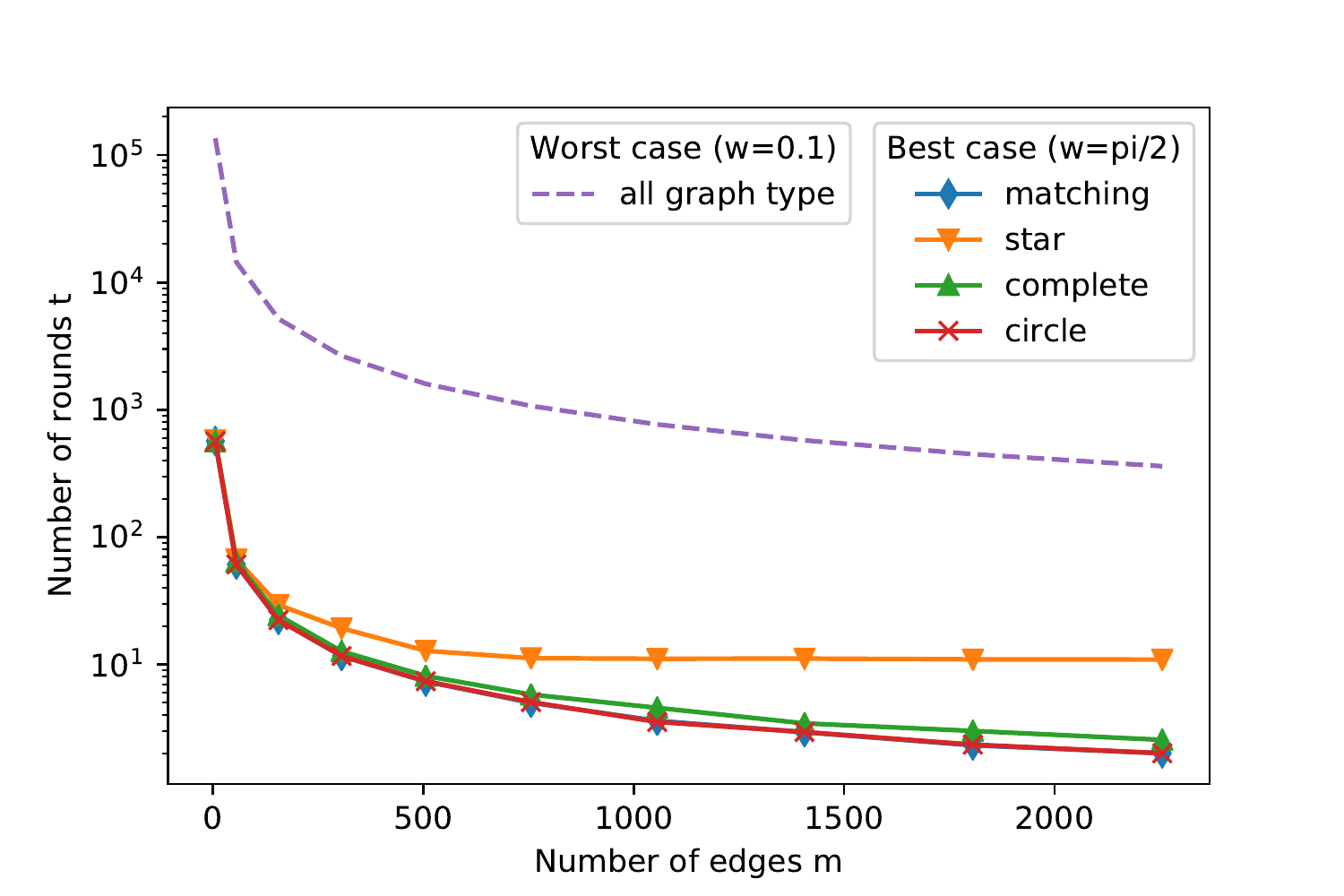}
}
%no space
\hfill
\subfloat{\label{fig:m}
\centering

\includegraphics[width=0.47\linewidth]{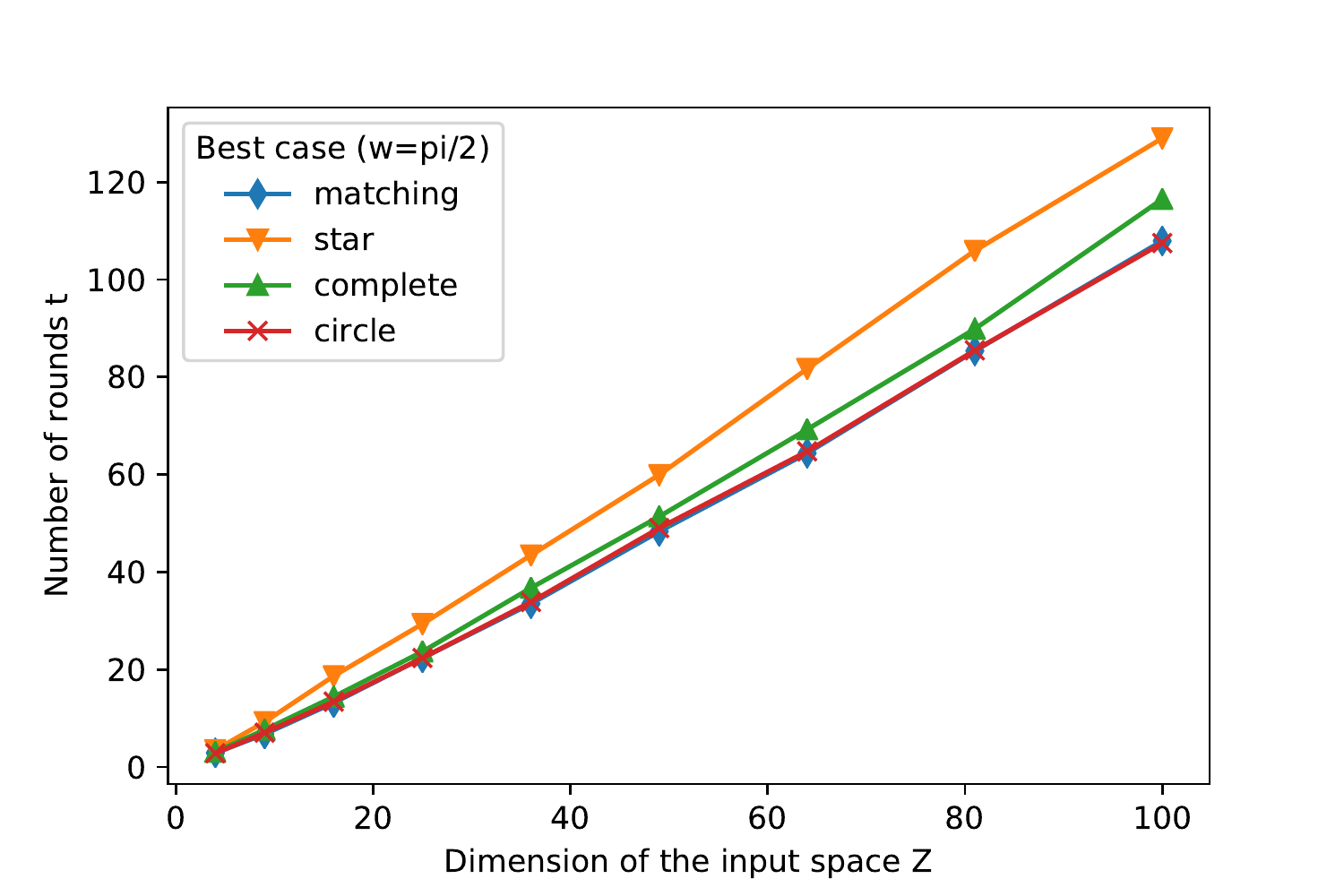}
}

\caption{Number of rounds $t$ needed to verify the stopping condition \eqref{eq:stopping_condition} with respect to  \textbf{left:} the number of edges $m$ where the dimension of the edge-arm space $\mathcal{Z}$ is fixed and equal to $25$ and \textbf{right:} the dimension of the edge-arm space $\mathcal{Z}$  where the number of edges is fixed and equal to $156$. For both experiments we run 100 times and plot the average number of rounds needed to verify the stopping condition.}
\label{fig:exp_d_and_m}
\end{figure*}

Hence, one has
\begin{align}
    \var\left(\mathbf{A}_1\right) =& \sum_{(i,j)\in E} \var\left(\mathbf{A}^{(i,j)}_1\right) \notag\\ 
    &+ \sum_{i=1}^{n} \sum_{j \in \mathcal{N}(i)} \sum_{\substack{k \in \mathcal{N}(i)\\ k\neq j}}\cov(\mathbf{A}^{(i,j)}_1, \mathbf{A}^{(i,k)}_1) \notag \\
    &+ \sum_{i=1}^{n} \sum_{j \in \mathcal{N}(i)} \sum_{k \in \mathcal{N}(j)}\cov(\mathbf{A}^{(i,j)}_1, \mathbf{A}^{(j,k)}_1) \notag \enspace.
\end{align}

Let $P \geq 0$ be such that for all $(i,j) \in E$,  $\var\left(\mathbf{A}^{(i,j)}_1\right) \preceq P\times \mathbf{I}$ and $M, N \geq 0$ such that for all $(i,j) \in E$:
\begin{align*}
&\forall k \in \mathcal{N}(i),  \cov(\mathbf{A}^{(i,j)}_1, \mathbf{A}^{(i,k)}_1) \preceq M \times \mathbf{I} \\
&\forall k \in \mathcal{N}(j), \cov(\mathbf{A}^{(i,j)}_1, \mathbf{A}^{(j,k)}_1) \preceq N \times \mathbf{I}
\end{align*}

We want to compare the quantity $\|\var(\mathbf{A}_1)\|$ for different types of graphs: star, complete, circle and a matching graph. To have a fair comparison, we want graphs that reveal the same number of rewards at each round of the learning procedure. Hence, we denote respectively $n_{\mathrm{S}}$, $n_{\mathrm{Co}}$, $n_{\mathrm{Ci}}$ and $n_{\mathrm{M}}$ the number of nodes in a star, complete, circle and matching graph of $m$ edges and get:

\textbf{Star graph:} 
\begin{align*} 
\|\var(\mathbf{A}_1)\| \leq m P + n_{\mathrm{S}}^2 (M + N).
\end{align*}
\textbf{Complete graph:}
\begin{align*} 
\|\var(\mathbf{A}_1)\| \leq m P + n_{\mathrm{Co}}^3( M + N).
\end{align*}
\textbf{Circle graph:}
\begin{align*}
    \|\var(\mathbf{A}_1)\| \leq m P + n_{\mathrm{Ci}}(2M + 4N).
\end{align*}
\textbf{Matching graph:}
\begin{align*}
    \|\var(\mathbf{A}_1)\| \leq m P + n_{\mathrm{M}} N.
\end{align*}

We refer the reader to Appendix \ref{sec:supp_variance} for more details on the given upper bounds.
Since the star (respectively, complete, circle and matching) graph of $m$ edges has a number of nodes $n_{\mathrm{S}} = m/2 + 1$ (respectively $n_{\mathrm{Co}} = \left(1 + \sqrt{4m +1}\right)/2$, $n_{\mathrm{Ci}} = m/2$ and $n_{\mathrm{M}} = m$), we obtain the bounds stated in Table~\ref{table:convergence_type_graph}.

\begin{table}[h]\centering
\ra{1.3}
\begin{tabular}{@{}lcc@{}}\toprule
    Graph &  Upper bound on $\|\var{\left(\mathbf{A}_1\right)}\|$ & $\alpha$  \\ \midrule
    Star & $mP + (M + N) O\left(m^2\right)$ & $O\left(1/\sqrt{t}\right)$\\ 
    Complete & $mP + (M + N)O\left(m\sqrt{m}\right)$ & $O\left(1/\left(m^{\frac{1}{4}}\sqrt{t}\right)\right)$\\
    Circle & $mP + (M + N)O\left(m\right)$ & $O\left(1/\sqrt{mt}\right)$\\
    Matching & $mP + mN$ & $O\left(1/\sqrt{mt}\right)$\\\bottomrule
\end{tabular}
\caption{Upper bound on the variance and convergence rate of Algorithm~\ref{algorithm:Randomized_G_allocation_for_GBB} for the star, complete, circle and matching graph with respect to the number of edges $m$ and the number of rounds $t$.}
\label{table:convergence_type_graph}
\end{table}

% Hence, for those tree types of graph, one can formulate the dependence of the convergence rate with respect to the number of edges $m$ and the number of round $t$ :
% \begin{table}[h]
% \begin{center}
% \begin{tabular}{ c | c  }

%     & $\alpha$  \\
%     \Xhline{2\arrayrulewidth}
%     star & $O\left(1/\sqrt{t}\right)$\\ 
%     \hline
%     complete & $O\left(1/\left(m^{\frac{1}{4}}\sqrt{t}\right)\right)$\\
%     \hline
%     circle & $O\left(1/\sqrt{mt}\right)$
% \end{tabular}
% \end{center}
% \caption{Convergence rate of Algorithm~\ref{algorithm:Randomized_G_allocation_for_GBB} for the star, complete and circle graph with respect to number of edges $m$ and rounds $t$}
% \label{table:convergence_type_graph}
% \end{table}
 
These four examples evidence the strong dependency of the variance on the structure of the graph. The more independent the edges are (\ie with no common nodes), the smaller the quantity $\|\var(\mathbf{A}_1)\|$ is. For a fixed number of edges $m$, the best case is the matching graph where no edge share the same node and the worst case is the star graph where all the edges share a central node.

\section{Experiments}
% \geovani{Toute cette partie a reprendre}
\label{sec:expe}

In this section, we consider the modified version of a standard experiment introduced by \citet{SoareNIPS2014} and used in most papers on best arm identification in linear bandits \citep{XuAISTATS2018,TaoICML2018,tanner2019transductive,Zaki2019GLUCB} to evaluate the sample complexity of our algorithm on different graphs.
% We evaluate the number of rounds needed to verify the stopping condition with respect to the type of graph, the number of edges $m$ and the edge-arm space dimension.
We consider $d+1$ node-arms in $\mathcal{X} \subset \mathbb{R}^d$ where $d \geq 2$. This node-arm set is made of the $d$ vectors $(\mathbf{e}_1, \dots, \mathbf{e}_d)$ forming the canonical basis of $\mathbb{R}^d$ and one additional arm $x_{d+1} = (\cos(\omega), \sin(\omega), 0, \dots, 0)^{\top}$ with $\omega \in ]0,\pi/2]$. Note that by construction, the edge-arm set $\mathcal{Z}$ contains the canonical basis $(\mathbf{e}^\prime_1, \dots, \mathbf{e}^\prime_{d^2})$ of $\mathbb{R}^{d^2}$. The parameter matrix $\mathbf{M}_{\star}$ has its first coordinate equal to 2 and the others equal to 0 which makes $\theta_\star = \vect{\left(\mathbf{M}_\star\right)} = \left(2, 0, \dots, 0\right)^\top \in \mathbb{R}^{d^2}$. The best edge-arm is thus $z_\star = z^{(1,1)} = \mathbf{e}^\prime_1$. One can note that when $\omega$ tends to 0, it is harder to differentiate this arm from $z^{(d+1,d+1)} = \vect{\left(x_{(d+1)}x_{(d+1)}^\top\right)}$than from the other arms. We set $\eta^{(i,j)}_t \sim \mathcal{N}(0,1)$, for all edges $(i,j)$ and round $t$.

We consider the two cases where $\omega=0.1$ which makes the edge-arms $z^{(1,1)}$ and $z^{(d+1,d+1)}$ difficult to differentiate, and $\omega=\pi/2$ which makes the edge-arm $z^{(1,1)}$ easily identifiable as the optimal edge-arm.
For each of these two cases, we evaluate the influence of the graph structure, the number of edges $m$ and the edge-arm space dimension $d^2$ on the sampling complexity. Results are shown in Figure~\ref{fig:exp_d_and_m}. 

When $\omega = 0.1$, the type of the graph does not impact the number of rounds needed to verify the stopping condition. This is mainly due to the fact that the magnitude of its associated variance is negligible with respect to the number of rounds. Hence, even if we vary the number of edges or the dimension, we get the same performance for any type of graph including the matching graph. This implies that our algorithm performs as well as a linear bandit that draws $m$ edge-arms in parallel at each round. When $\omega=\pi/2$, the number of rounds needed to verify the stopping condition is smaller and the magnitude of the variance is no longer negligible. Indeed, when the number of edges or the dimension increases, we notice that the star graph takes more times to satisfy the stopping condition. Moreover, note that the sample complexities obtained for the circle and the matching graph are similar. This observation is in line with the dependency on the variance shown in Table~\ref{table:convergence_type_graph}.

\section{Conclusion}
We introduced a new graphical bilinear bandit setting and studied the best arm identification problem with a fixed confidence. This problem being NP-Hard even with the knowledge of the true parameter matrix $\mathbf{M}^\star$, we first proposed an algorithm that provides a 1/2-approximation. Then, we provided a second algorithm, based on G-allocation strategy, that uses randomized sampling over the nodes to return a good estimate $\hat{\mathbf{M}}$ that can be used instead of $\mathbf{M}_\star$.
%We characterised the links and the constraints of our problem according to the linear bandit setting. 
% To identify the best edge-arm, we proposed a decentralized algorithm which constructs an estimate of the reward parameter by using a random sampling strategy, pulling the optimal arms at each node. Then, given the obtained estimate, we proposed a second algorithm that provides a 1/2-approximation of the associated NP-hard allocation problem. 
Finally, we highlighted the impact of the graph structure on the convergence rate of our algorithm and validated our theoretical results with experiments. Promising extensions of the model include considering unknown parameters $\mathbf{M}_\star^{(i,j)}$, different for each edge $(i,j)$ of the graph, and investigating XY-allocation strategies.

\section*{Acknowledgements}

We thank anonymous reviewers, whose comments helped
us improve the paper significantly.

\newpage

\bibliography{biblio}

\begin{thebibliography}{}

\bibitem[Abbasi-Yadkori et~al., 2011]{abbasi2011improved}
Abbasi-Yadkori, Y., P{\'a}l, D., and Szepesv{\'a}ri, C. (2011).
\newblock Improved algorithms for linear stochastic bandits.
\newblock In {\em Advances in Neural Information Processing Systems}, pages
  2312--2320.

\bibitem[Amin et~al., 2011]{amin2012graphical}
Amin, K., Kearns, M., and Syed, U. (2011).
\newblock Graphical models for bandit problems.
\newblock In {\em Proceedings of the Twenty-Seventh Conference on Uncertainty
  in Artificial Intelligence}, page 1–10.

\bibitem[Audibert and Bubeck, 2010]{audibert2010best}
Audibert, J.-Y. and Bubeck, S. (2010).
\newblock Best arm identification in multi-armed bandits.
\newblock In {\em Proceedings of the 23th Annual Conference on Learning
  Theory}, pages 41--53.

\bibitem[Audibert et~al., 2011]{audibert2011semibandits}
Audibert, J.-Y., Bubeck, S., and Lugosi, G. (2011).
\newblock Minimax policies for combinatorial prediction games.
\newblock In {\em Proceedings of the 24th Annual Conference on Learning
  Theory}, pages 107--132.

\bibitem[Bargiacchi et~al., 2018]{bargiacchi2018windfarms}
Bargiacchi, E., Verstraeten, T., Roijers, D., Now{\'e}, A., and Hasselt, H.
  (2018).
\newblock Learning to coordinate with coordination graphs in repeated
  single-stage multi-agent decision problems.
\newblock In {\em International conference on machine learning}, pages
  482--490.

\bibitem[Bubeck et~al., 2009]{bubeck2009pure}
Bubeck, S., Munos, R., and Stoltz, G. (2009).
\newblock Pure exploration in multi-armed bandits problems.
\newblock In {\em International conference on Algorithmic learning theory},
  pages 23--37. Springer.

\bibitem[Cao and Krishnamurthy, 2019]{Cao2019}
Cao, T. and Krishnamurthy, A. (2019).
\newblock Disagreement-based combinatorial pure exploration: Sample complexity
  bounds and an efficient algorithm.
\newblock In {\em Proceedings of the Thirty-Second Conference on Learning
  Theory}, volume~99, pages 558--588.

\bibitem[Cesa-Bianchi et~al., 2013]{cesa2013gang}
Cesa-Bianchi, N., Gentile, C., and Zappella, G. (2013).
\newblock A gang of bandits.
\newblock In {\em Advances in Neural Information Processing Systems}, pages
  737--745.

\bibitem[Cesa-Bianchi and Lugosi, 2012]{cesabianchi2012}
Cesa-Bianchi, N. and Lugosi, G. (2012).
\newblock Combinatorial bandits.
\newblock {\em Journal of Computer and System Sciences}, 78(5):1404 -- 1422.

\bibitem[Chen et~al., 2014]{Chen2014}
Chen, S., Lin, T., King, I., Lyu, M.~R., and Chen, W. (2014).
\newblock Combinatorial pure exploration of multi-armed bandits.
\newblock In {\em Advances in Neural Information Processing Systems},
  volume~27, pages 379--387.

\bibitem[Chen et~al., 2013]{chen2013combinatorial}
Chen, W., Wang, Y., and Yuan, Y. (2013).
\newblock Combinatorial multi-armed bandit: General framework and applications.
\newblock In {\em International Conference on Machine Learning}, pages
  151--159.

\bibitem[Damla~Ahipasaoglu et~al., 2008]{damla2008linear}
Damla~Ahipasaoglu, S., Sun, P., and Todd, M.~J. (2008).
\newblock Linear convergence of a modified frank--wolfe algorithm for computing
  minimum-volume enclosing ellipsoids.
\newblock {\em Optimisation Methods and Software}, 23(1):5--19.

\bibitem[Degenne et~al., 2020]{degenne2020gamification}
Degenne, R., M{\'e}nard, P., Shang, X., and Valko, M. (2020).
\newblock Gamification of pure exploration for linear bandits.
\newblock {\em arXiv preprint arXiv:2007.00953}.

\bibitem[Du et~al., 2020]{du2020combinatorial}
Du, Y., Kuroki, Y., and Chen, W. (2020).
\newblock Combinatorial pure exploration with full-bandit or partial linear
  feedback.
\newblock {\em arXiv e-prints}, pages arXiv--2006.

\bibitem[Erdos, 1975]{erdos1975problems}
Erdos, P. (1975).
\newblock Problems and results on finite and infinite graphs.
\newblock In {\em Recent advances in graph theory (Proc. Second Czechoslovak
  Sympos., Prague, 1974)}, pages 183--192.

\bibitem[Fiez et~al., 2019]{tanner2019transductive}
Fiez, T., Jain, L., Jamieson, K.~G., and Ratliff, L. (2019).
\newblock Sequential experimental design for transductive linear bandits.
\newblock In {\em Advances in Neural Information Processing Systems}, pages
  10667--10677.

\bibitem[Frank et~al., 1956]{frank1956algorithm}
Frank, M., Wolfe, P., et~al. (1956).
\newblock An algorithm for quadratic programming.
\newblock {\em Naval research logistics quarterly}, 3(1-2):95--110.

\bibitem[Guestrin et~al., 2002]{Guestrin2002coordgraph}
Guestrin, C., Lagoudakis, M.~G., and Parr, R. (2002).
\newblock Coordinated reinforcement learning.
\newblock In {\em Proceedings of the Nineteenth International Conference on
  Machine Learning}, page 227–234.

\bibitem[Hassin and Khuller, 2001]{hassin2001z}
Hassin, R. and Khuller, S. (2001).
\newblock z-approximations.
\newblock {\em Journal of Algorithms}, 41(2):429--442.

\bibitem[Jedra and Proutiere, 2020]{jedra2020optimal}
Jedra, Y. and Proutiere, A. (2020).
\newblock Optimal best-arm identification in linear bandits.
\newblock {\em arXiv preprint arXiv:2006.16073}.

\bibitem[Jourdan et~al., 2021]{Jourdan2021}
Jourdan, M., Mut\'y, M., Kirschner, J., and Krause, A. (2021).
\newblock Efficient pure exploration for combinatorial bandits with semi-bandit
  feedback.
\newblock In {\em Proceedings of the 31st International Conference on
  Algorithmic Learning Theory}.

\bibitem[Jun et~al., 2019]{jun2019bilinear}
Jun, K.-S., Willett, R., Wright, S., and Nowak, R. (2019).
\newblock Bilinear bandits with low-rank structure.
\newblock In {\em International Conference on Machine Learning}, pages
  3163--3172.

\bibitem[Kazerouni and Wein, 2019]{kazerouni2019best}
Kazerouni, A. and Wein, L.~M. (2019).
\newblock Best arm identification in generalized linear bandits.
\newblock {\em arXiv preprint arXiv:1905.08224}.

\bibitem[Kiefer and Wolfowitz, 1960]{KieferEquivalenceThm1960}
Kiefer, J. and Wolfowitz, J. (1960).
\newblock The equivalence of two extremum problems.
\newblock {\em Canadian Journal of Mathematics}, 12:363--366.

\bibitem[Mannor and Shamir, 2011]{mannor2011bandits}
Mannor, S. and Shamir, O. (2011).
\newblock From bandits to experts: On the value of side-observations.
\newblock In {\em Advances in Neural Information Processing Systems}, pages
  684--692.

\bibitem[Perrault et~al., 2020]{Perrault2020}
Perrault, P., Boursier, E., Valko, M., and Perchet, V. (2020).
\newblock Statistical efficiency of thompson sampling for combinatorial
  semi-bandits.
\newblock In {\em Advances in Neural Information Processing Systems}.

\bibitem[Petersen and Pedersen, 2012]{petersen2012matrix}
Petersen, K.~B. and Pedersen, M.~S. (2012).
\newblock The matrix cookbook, nov 2012.
\newblock {\em URL http://www2. imm. dtu. dk/pubdb/p. php}, 3274:14.

\bibitem[Pukelsheim, 2006]{Pukelsheim2006}
Pukelsheim, F. (2006).
\newblock {\em Optimal Design of Experiments}.
\newblock Society for Industrial and Applied Mathematics.

\bibitem[Rizk et~al., 2019]{rizk2020refined}
Rizk, G., Colin, I., Thomas, A., and Draief, M. (2019).
\newblock Refined bounds for randomized experimental design.
\newblock {\em NeurIPS Workshop on Machine Learning with Guarantees}.

\bibitem[Robbins, 1952]{robbins1952some}
Robbins, H. (1952).
\newblock Some aspects of the sequential design of experiments.
\newblock {\em Bulletin of the American Mathematical Society}, 58(5):527--535.

\bibitem[Sagnol, 2010]{SagnolPhd2010}
Sagnol, G. (2010).
\newblock {\em Optimal design of experiments with application to the inference
  of traffic matrices in large networks: second order cone programming and
  submodularity}.
\newblock PhD thesis, École Nationale Supérieure des Mines de Paris.

\bibitem[{Siomina} et~al., 2006]{siomina2006wireless}
{Siomina}, I., {Varbrand}, P., and {Yuan}, D. (2006).
\newblock Automated optimization of service coverage and base station antenna
  configuration in {UMTS} networks.
\newblock {\em IEEE Wireless Communications}, 13(6):16--25.

\bibitem[Soare et~al., 2014]{SoareNIPS2014}
Soare, M., Lazaric, A., and Munos, R. (2014).
\newblock Best-arm identification in linear bandits.
\newblock In {\em Advances in Neural Information Processing Systems}, pages
  828--836.

\bibitem[Tao et~al., 2018]{TaoICML2018}
Tao, C., Blanco, S., and Zhou, Y. (2018).
\newblock Best arm identification in linear bandits with linear dimension
  dependency.
\newblock In {\em Proceedings of the 35th International Conference on Machine
  Learning}, volume~80, pages 4877--4886.

\bibitem[Tropp et~al., 2015]{tropp2015introduction}
Tropp, J.~A. et~al. (2015).
\newblock An introduction to matrix concentration inequalities.
\newblock {\em Foundations and Trends{\textregistered} in Machine Learning},
  8(1-2):1--230.

\bibitem[Valko, 2020]{valko2020bandits}
Valko, M. (2020).
\newblock Bandits on graphs and structures.

\bibitem[Valko et~al., 2014]{valko2014spectralbandits}
Valko, M., Munos, R., Kveton, B., and Koc{\'a}k, T. (2014).
\newblock Spectral bandits for smooth graph functions.
\newblock In Xing, E.~P. and Jebara, T., editors, {\em International conference
  on machine learning}, volume~32 of {\em Proceedings of Machine Learning
  Research}, pages 46--54.

\bibitem[Welch, 1982]{WelchNPHard}
Welch, W. (1982).
\newblock Algorithmic complexity: Three np-hard problems in computational
  statistics.
\newblock {\em Journal of Statistical Computation and Simulation - J STAT
  COMPUT SIM}, 15:17--25.

\bibitem[Xu et~al., 2018]{XuAISTATS2018}
Xu, L., Honda, J., and Sugiyama, M. (2018).
\newblock A fully adaptive algorithm for pure exploration in linear bandits.
\newblock In {\em Proceedings of the Twenty-First International Conference on
  Artificial Intelligence and Statistics}, volume~84, pages 843--851.

\bibitem[Zaki et~al., 2019]{Zaki2019GLUCB}
Zaki, M., Mohan, A., and Gopalan, A. (2019).
\newblock Towards optimal and efficient best arm identification in linear
  bandits.
\newblock {\em arXiv preprint arXiv:1911.01695}.

\bibitem[Zaki et~al., 2020]{zaki2020explicit}
Zaki, M., Mohan, A., and Gopalan, A. (2020).
\newblock Explicit best arm identification in linear bandits using no-regret
  learners.
\newblock {\em arXiv preprint arXiv:2006.07562}.

\bibitem[Çivril and Magdon-Ismail, 2009]{CivrilNPHard}
Çivril, A. and Magdon-Ismail, M. (2009).
\newblock On selecting a maximum volume sub-matrix of a matrix and related
  problems.
\newblock {\em Theoretical Computer Science}, 410(47):4801 -- 4811.

\end{thebibliography}

\onecolumn
\newpage
\appendix
\section{An NP-Hard Problem}
\label{sec:supp_finding_best_arm}
\subsection{Proof of Theorem~\ref{theorem:nphard}}
\begin{theorem}
Consider a given matrix $\mathbf{M}_\star \in \mathbb{R}^{d \times d}$ and a finite arm set $\mathcal{X} \subset \mathbb{R}^d$. Unless P=NP, there is no polynomial time algorithm guaranteed to find the optimal solution of 
\begin{align*}
\max_{\left(x^{(1)}, \ldots, x^{(n)}\right) \in\mathcal{X}^n} \sum_{(i,j)\in E} x^{(i)\top}\mathbf{M}_\star \ x^{(j)}\enspace.
\end{align*}
\end{theorem}

\begin{proof}
We prove the statement by reduction to the Max-Cut problem. Let $\mathcal{G}=(V,E)$ be a graph with $V=\{1, \ldots, n\}$. Let $\mathcal{X}=\left\{ e_{0},e_{1}\right\} $,
where $e_{0}=\left(1,0\right)^{\top}$ and $e_{1}=\left(0,1\right)^{\top}$.
Let $\mathbf{M}_\star=\left[\begin{array}{cc}
0 & 1\\
1 & 0
\end{array}\right]$. 
For any joint arm assignment $\left(x^{(1)}\ldots x^{(n)}\right)\in\mathcal{X}^{n}$, let $F\subseteq E$ be defined as $F=\left\{ i:x^{(i)}=e _{1}\right\}$. Note that
\[
\sum_{(i,j) \in E}x^{(i)\top} \mathbf{M}_\star x^{(j)}=\sum_{(i,j) \in E}\mathbf{1}\left[x^{(i)} \neq x^{(j)}\right]=2\times\sum_{(i,j) \in E}\mathbf{1}\left[i\in F,j\notin F\right],
\]
where $\mathbf{1}[\cdot]$ is the indicator function. The assignement $\left(x^{(1)},\ldots, x^{(n)}\right)$ induces a cut $(F,V\backslash F)$,
and the value of the assignment is \emph{precisely} twice the value of the cut. Thus, if there was a polynomial time algorithm solving our problem, this algorithm would also solve the Max-Cut problem.
\end{proof}

\subsection{Proof of Theorem~\ref{theorem:halfOptimiality}}
\begin{theorem}
Let us consider the graph $\mathcal{G}=(V,E)$, a finite arm set $\mathcal{X} \subset \mathbb{R}^d$ and the matrix $\mathbf{M}_\star$ given as input to Algorithm~\ref{algorithm:bipartite}. Then, the expected global reward $r = \sum_{(i,j)\in E} x^{(i)\top} \mathbf{M}_\star x^{(j)}$ associated to the returned allocation $\mathbf{x}=\left(x^{(1)},\dots,x^{(n)}\right) \in \mathcal{X}^n$ verifies:
\begin{align*}
    \frac{r - r_{\mathrm{min}}}{r_{\star} - r_{\mathrm{min}}} \geq \frac{1}{2} \enspace.
\end{align*}

where $r_{\star}$ and $r_{\mathrm{min}}$ are respectively the highest and lowest global reward one can obtain with the appropriate joint arm.
Finally, the complexity of the algorithm is in $\mathcal{O}(K^2 + n)$.
\end{theorem}

\begin{proof}
 Given the matrix $\mathbf{M}_\star$, the algorithm obtains the two node-arms $(x_\star,x_\star^\prime) \in \mathcal{X}$ solution of
  \begin{align*}
      \max_{(x, x^\prime) \in \mathcal{X}} x^\top \mathbf{M}_\star x^\prime \enspace.
  \end{align*}
  Note that it is equivalent to obtain $z_\star$ solution of
  \begin{align*}
      \max_{z \in \mathcal{Z}} \; z^\top \vect{\left(\mathbf{M}_\star\right)} \enspace.
  \end{align*}
%   One can easily notice that the optimal gobal reward $r_{\star}$ is upper-bounded by $m\cdot x_\star^\top \mathbf{M}_\star x_\star^\prime$.
  Let us analyze a round of Algorithm~\ref{algorithm:bipartite} where we assign the arm of a node in $V$. For sake of simplicity, we assume that node $i$ is assigned at round $i$. At round $i$, we count the number $n_{1}^{(i)}$ of neighbors of $i$ that have already been assigned the arm $x_\star$ and we count the number $n_{2}^{(i)}$ of neighbors of $i$ that have already been assigned the arm $x_\star^\prime$. Then, node $i$ is assigned the arm least represented among its neighbors, that is arm $x_\star$ if $n_{2}^{(i)} \geq n_{1}^{(i)}$ and $x^\prime_\star$ otherwise. Eventually, the optimal edge-arm $z_\star$ has been assigned $\max(n_{1}^{(i)} , n_{2}^{(i)})$ times among node $i$'s neighborhood.
  %the arm $x_\star$ to node $i$ if there are more neighbors with assigned arm $x_\star^\prime$, or we assign to node $i$ the arm $x_\star^\prime$ if there are more neighbors with assigned arm $x_\star$. In all cases, we have assigned $n_{1}^{(i)} \vee n_{2}^{(i)}$ optimal arms $z_\star$ and $n_{1}^{(i)} \wedge n_{2}^{(i)}$ suboptimal arms $z\neq z_\star \in \mathcal{Z}$.
  Hence, for each node $i$, if we denote $r_i$ the sum of all the rewards obtained with the edge-arms constructed only during the round $i$, we have
  \begin{align*}
      r_i &= \max\left(n_{1}^{(i)}, n_{2}^{(i)}\right) z_\star^\top \theta_\star + \min\left(n_{1}^{(i)}, n_{2}^{(i)}\right) z^\top \theta_\star\\
      & \geq \frac{n_{1}^{(i)} + n_{2}^{(i)}}{2} ( z_\star^\top \theta_\star  + z^\top \theta_\star) \enspace .
  \end{align*}
  One can notice that the arm $z$ can only be equal to $\vect{\left(x_\star x_\star^{\top}\right)}$ or $\vect{\left(x_\star^{\prime\top} x_\star^{\prime\top}\right)}$. Let assume that $\vect{\left(x_\star x_\star^{\top}\right)}^\top \theta_\star \leq \vect{\left(x_\star^{\prime} x_\star^{\prime\top}\right)}^\top \theta_\star$ without loss of generality and let consider the worst case where $z$ is always equal to $\vect{\left(x_\star x_\star^{\top}\right)}$. Since z is constructed with the same node-arm $x_\star$, the allocation that constructs at each edge the edge-arm $z$ exists (which is allocating $x_\star$ to all the nodes), thus $m \times z^\top \theta_\star \geq r_{\mathrm{min}}$.
  
  Moreover one can also notice that $m \times z_\star^\top \theta_\star \geq r_\star$. We thus have, 
  
  \begin{align*}
    r_i & \geq \frac{n_{1}^{(i)} + n_{2}^{(i)}}{2m} (r_\star + r_{\mathrm{min}}) \enspace.
  \end{align*}
  Now let us sum all the rewards obtained with the constructed edge-arms at each round of the algorithm, that is the global reward $r$ of the graph allocation returned by the proposed algorithm:
  \begin{align*}
      r &= \sum_{i = 1}^{n} r_i \\
      & \geq \sum_{i = 1}^{n} \frac{n_{1}^{(i)} + n_{2}^{(i)}}{2m} (r_\star + r_{\mathrm{min}}) \\
      & = \frac{1}{2} (r_\star + r_{\mathrm{min}}) \\
      & = \frac{1}{2} (r_\star  - r_{\mathrm{min}}) + r_{\mathrm{min}} \enspace.
\end{align*}

Moreover, the algorithm does $K^2$ estimation to find the best couple $(x_\star, x_\star^\prime) \in \mathcal{X}^2$, and each of the $n$ rounds of the algorithm is in $O(1)$. Hence the complexity is equal to $O(K^2 + n)$.
\end{proof}

\section{Deriving the stopping condition}
\label{sec:supp_stopping_condition}
In this section, we remind key results to derive the stopping condition. We refer the reader to \citet{SoareNIPS2014} and references therein for additional details.
Let $\mathcal{Z} \subset \bbR^{d^2}$ be the set of edge-arms and let $K^2 = |\mathcal{Z}|$. For $m, t > 0$, we consider a sequence of edge-arms $\mathbf{z}_t = (z_{1}, \dots, z_{mt}) \in \mathcal{Z}^{mt}$ and the corresponding noisy rewards $(r_{1}, \dots, r_{mt})$. We assume that the noise terms in the rewards are i.i.d., following a $\sigma$-sub-Gaussian distribution. Let $\hat{\theta}_t = \mathbf{A}_{t}^{-1} b_t \in \mathbb{R}^{d^2}$ be the solution of the ordinary least squares problem with $\mathbf{A}_t = \sum_{i=1}^{mt} z_{i} z_{i}^{\top} \in \mathbb{R}^{d^2 \times d^2}$ and $b_t = \sum_{i=1}^{k} z_{i} r_{i} \in \mathbb{R}^{d^2}$. We first recall the following property.
\begin{prop}[Proposition 1 in \citet{SoareNIPS2014}]
\label{prop_soare}
Let $c = 2\sigma\sqrt{2}$. For every fixed sequence $\mathbf{z}_{t}$, with probability $1 - \delta$, for all $t > 0$ and for all $z \in \mathcal{Z}$, we have
\begin{align*}
    \left|z^\top \theta_{\star} - z^\top \hat{\theta}_t\right| \leq c\|z\|_{\mathbf{A}_t^{-1}} \sqrt{\log\left(\frac{6m^2 t^{2} K^2}{\delta \pi} \right)} \enspace.
\end{align*}
\end{prop}

Our goal is to find the arm $z_\star$ that has the optimal expected reward $z_\star^\top \theta_\star$. In other words, we want to find an arm $z \in \mathcal{Z}$, such that for all $z^\prime \in \mathcal{Z}$, $(z - z^\prime)^\top \theta_\star \geq 0$.
However, one does not have access to $\theta_\star$, so we have to use its empirical estimate.

Let us consider a confidence set $\hat{S}(\mathbf{z}_t)$ centered at $\hat{\theta}_t \in \hat{S}(\mathbf{z}_t)$ and such that $\mathbb{P}\left(\theta_{\star}\notin \hat{S}(\mathbf{z}_t) \right) \leq \delta$, for some $\delta > 0$. Since $\theta_{\star}$ belongs to $\hat{S}(\mathbf{z}_t)$ with probability at least $1 - \delta$, one can stop pulling arms when an arm has been found, such that the above condition is verified for any $\theta \in \hat{S}(\mathbf{z}_t)$. More formally, the best arm identification task will be considered successful when an arm $z \in \mathcal{Z}$ will verify the following condition for any $z^\prime \in \mathcal{Z}$ and any $\theta \in \hat{S}(\mathbf{z}_t)$:
\[
    (z - z^\prime)^\top (\hat{\theta}_t - \theta) \leq \hat{\Delta}_t(z, z^\prime) \enspace,
\]
where $\hat{\Delta}_t\left(z,z^\prime\right) = \left(z - z^\prime\right)^{\top} \hat{\theta}_t$ is the empirical gap between $z$ and $z^\prime$.
% \begin{align*}
%     & \ \ \ \exists z \in \mathcal{Z}, \forall z^\prime \in \mathcal{Z}, \forall \theta \in \hat{S}(\mathbf{z}_t),  (z-z^\prime)^\top \theta_\star \geq 0 \\
%     \Leftrightarrow & \ \ \ \exists z \in \mathcal{Z}, \forall z^\prime \in \mathcal{Z}, \forall \theta \in \hat{S}(\mathbf{z}_t),  (z-z^\prime)^\top \left(\hat{\theta}_t - \theta\right) \leq \hat{\Delta}_t(z,z^\prime)\enspace.
% \end{align*}

Using the upper bound in Proposition~\ref{prop_soare}, one way to ensure that $\mathbb{P}\left(\theta_{\star}\in \hat{S}(\mathbf{z}_t) \right)\geq 1 - \delta$ is to define the confidence set $\hat{S}(\mathbf{z}_t)$ as follows
\begin{align*}
    \hat{S}\left(\mathbf{z}_t\right) = \left\{ \theta \in \mathbb{R}^d, \ \forall z \in \mathcal{Z}, \ \forall z^\prime \in \mathcal{Z}, \left(z - z^{\prime}\right)^{\top}\left(\hat{\theta}_t - \theta \right) \leq c\|z - z^{\prime}\|_{\left(\mathbf{A}_t\right)^{-1}} \sqrt{\log\left(\frac{6 m^{2} t^{2} K^4}{ \delta \pi }\right)} \right\} \enspace.
\end{align*}
Then, the stopping condition can be reformulated as follows:
\begin{align}
\exists z \in \mathcal{Z}, \ \forall z^\prime \in \mathcal{Z}, \  c\|z - z^{\prime}\|_{\mathbf{A}_t^{-1}} \sqrt{\log\left(\frac{6 m^{2} t^{2} K^4}{ \delta \pi }\right)} \leq \hat{\Delta}_t\left(z,z^\prime\right) \label{stoppingcondition1}  \enspace.
\end{align}

\section{Estimation of the unknown parameter}
\label{sec:supp_parameter_estimation}
\subsection{Proof of Theorem~\ref{theorem:mu_to_lambda}}
To prove Theorem~\ref{theorem:mu_to_lambda}, we first state some useful propositions and lemmas.
For any finite set $X \subset \mathbb{R}^d$, we define the function $h_X: \mathcal{S}_{X} \rightarrow \mathbb{R} \cup \{+\infty\}$ as follows: for any $\lambda \in \mathcal{S}_{X}$,
\[
h_X(\lambda) =
\begin{cases}
\max_{x^\prime \in X} x^{\prime \top} \Sigma_X(\lambda)^{-1} x^\prime & \text{if } \Sigma_X(\lambda) \text{ is invertible} \\
+\infty & \text{otherwise} \enspace.
\end{cases}
\]

\begin{lemma}
\label{lemma_mu_to_lambda_2}
Let $\mathcal{X} \subset \bbR^d$ be a finite set spanning $\bbR^d$ and let $\mathcal{Z} = \{ \vect{(x x^{\prime\top})} \text{, } (x, x^\prime) \in \mathcal{X}^2 \}$. If $\mu^\star \in \mathcal{S}_{\mathcal{X}}$ is a minimizer of $h_{\mathcal{X}}$, then $\mu^\star$ is a solution of
\begin{align*}
    \min_{\mu \in \mathcal{S}_{\mathcal{X}}} \max_{z \in \mathcal{Z}} z^\top \left(\sum_{x \in \mathcal{X}} \sum_{x^\prime \in \mathcal{X}} \mu_{x} \mu_{x^\prime}  \vect{\left( xx^{\prime\top}\right)}\vect{\left( xx^{\prime\top}\right)}^\top \right)^{-1} z \label{relaxed_G_allocation_X} \enspace.
\end{align*}
\end{lemma}

\begin{proof}
First, let us notice that, for any $\mathcal{X} \subset \bbR^d$, one has $h_{\mathcal{X}} \geq 0$. Thus, $\mu^\star$ is also a minimizer of $h_{\mathcal{X}}^2$. In addition, $\mathcal{X}$ is spanning $\bbR^d$ so $h_{\mathcal{X}}(\mu^\star) < +\infty$. Developing $h_{\mathcal{X}}(\mu^\star)^2$ yields:
\begin{align*}
    h_{\mathcal{X}}(\mu^\star) \times h_{\mathcal{X}}(\mu^\star)
    &= \left(\max_{x \in \mathcal{X}} \; x^{\top} \Sigma_X(\mu^\star)^{-1} x\right) \times \left(\max_{x \in \mathcal{X}} \; x^{\top} \Sigma_X(\mu^\star)^{-1}  x\right) \\
    &= \max_{x \in \mathcal{X}} \max_{x^{\prime} \in \mathcal{X}} \; x^{\top} \Sigma_X(\mu^\star)^{-1} x x^{\prime\top} \Sigma_X(\mu^\star)^{-1} x^{\prime} \\
    &= \max_{x \in \mathcal{X}} \max_{x^{\prime} \in \mathcal{X}} \; \vect{\left(x x^{\prime\top}\right)}^\top \vect{\left( \Sigma_X(\mu^\star)^{-1} x x^{\prime\top} \Sigma_X(\mu^\star)^{-1} \right)} \\
    &= \max_{x \in \mathcal{X}} \max_{x^{\prime} \in \mathcal{X}} \; \vect{\left(x x^{\prime\top}\right)}^\top \left( \Sigma_X(\mu^\star)^{-1} \otimes \Sigma_X(\mu^\star)^{-1} \right) \vect{\left(x x^{\prime\top}\right)} \\
    &= \max_{z \in \mathcal{Z}} \; z^\top \left( \Sigma_X(\mu^\star)^{-1} \otimes \Sigma_X(\mu^\star)^{-1} \right) z \enspace,
\end{align*}
where $\otimes$ denotes the Kronecker product. We can now focus on the central term:
\begin{align*}
    \Sigma_X(\mu^\star)^{-1} \otimes \Sigma_X(\mu^\star)^{-1}
    &=\left(\sum_{x\in \mathcal{X}} \mu^\star_{x} x x^\top \right)^{-1} \otimes \left(\sum_{x \in \mathcal{X}} \mu^\star_{x} x x^{\top} \right)^{-1}\\
    &= \left(\sum_{x\in \mathcal{X}} \mu^\star_{x} x x^\top \otimes \sum_{x \in \mathcal{X}} \mu^\star_{x} x x^{\top} \right)^{-1}\\
    &= \left(\sum_{x\in \mathcal{X}} \sum_{x^\prime \in \mathcal{X}} \mu^\star_{x} \mu^\star_{x^\prime} \left(x x^\top \otimes x^\prime x^{\prime\top}\right) \right)^{-1}\\
    &= \left(\sum_{x\in \mathcal{X}} \sum_{x^\prime \in \mathcal{X}} \mu^\star_{x} \mu^\star_{x^\prime} \vect{\left(x x^{\prime\top} \right)} \vect{\left(x x^{\prime\top} \right)}^\top \right)^{-1} \enspace,
\end{align*}
and the result holds.
\end{proof}

\begin{theorem}
\label{theorem:mu_to_lambda_2}
Let $\mu^\star \in \mathcal{S}_{\mathcal{X}}$ be a minimizer of $h_{\mathcal{X}}$. Let $\lambda^\star \in \mathcal{S}_{\mathcal{Z}}$ be the distribution defined from $\mu^\star$ such that, for all $z = \vect{(x x^{\prime\top})}$, $\lambda^\star_z = \mu^\star_x \mu^\star_{x^\prime}$. Then $\lambda^\star$ is a minimizer of $h_{\mathcal{Z}}$.

% Let $\mu^\star$ be a solution of the following optimization problem: 
% \begin{align}
%     \min_{\mu \in \mathcal{S}_K} \max_{x^\prime\in \mathcal{X}} {x^\prime}^\top\left(\sum_{x \in \mathcal{X}} \mu_x x x^\top\right)^{-1}x^\prime \enspace.
% \end{align}
% Then, there exists $\lambda^\star$ solution of
% \begin{align}
%     \min_{\lambda \in \mathcal{S}_{K^2}} \max_{z^\prime\in \mathcal{Z}} {z^\prime}^\top\left(\sum_{z \in \mathcal{Z}} \lambda_z zz^\top\right)^{-1}z^\prime \label{eq:relaxed_G_allocation_Z} \enspace.
% \end{align}

% such that for all $z \in \mathcal{Z}$, where $z = \vect{\left(xx^{\prime\top}\right)}$ with $(x,x^\prime) \in \mathcal{X}^2$, $\lambda_z^\star = \mu_x^\star \times \mu_{x^\prime}^\star$.
\end{theorem}

\begin{proof}

From \citet{KieferEquivalenceThm1960}, we know that $\min_{\lambda \in \mathcal{S}_{\mathcal{Z}}} h_{\mathcal{Z}}(\lambda) = d^2$ and $\min_{\mu \in \mathcal{S}_{\mathcal{X}}} h_{\mathcal{X}}(\mu) = d$. Then, using Proposition~\ref{lemma_mu_to_lambda_2}, one has
\begin{align*}
    d^2 &= h_{\mathcal{X}}(\mu^\star) \times h_{\mathcal{X}}(\mu^\star) \\
    &= \max_{z\in \mathcal{Z}} \; z^\top \left(\sum_{x\in \mathcal{X}} \sum_{x^\prime \in \mathcal{X}} \mu^\star_{x} \mu^\star_{x^\prime}  \vect{\left(x x^{\prime\top}\right)} \vect{\left(x x^{\prime\top}\right)}^\top  \right)^{-1} z \enspace.
\end{align*}
This result implies that $h_{\mathcal{Z}}(\lambda^\star) = d^2$. Since $\min_{\lambda \in \mathcal{S}_{\mathcal{Z}}} h_Z(\lambda) = d^2$, $\lambda^\star$ is a minimizer of $h_{\mathcal{Z}}$.

\end{proof}

\subsection{Proof of Theorem~\ref{theorem:convergence}}

To prove our confidence bound, we need the two following proposition. The first one is from \cite{tropp2015introduction}.

\begin{prop}[\citet{tropp2015introduction}, Chapter 5 and 6]
  \label{prop:chernoff-matrix_2}
  Let $\mathbf{Z}_1, \ldots, \mathbf{Z}_t$ be i.i.d.\ positive semi-definite random matrices in $\bbR^{d^2 \times d^2}$, such that there exists $L > 0$ verifying $\mathbf{0} \preceq \mathbf{Z}_1 \preceq mL \mathbf{I}$. Let $\mathbf{A}_t$ be defined as
  $\mathbf{A}_t \triangleq \sum_{s = 1}^t \mathbf{Z}_s$. Then, for any $0 < \varepsilon < 1$, one can lowerbound $\lambda_{\mathrm{min}}(\mathbf{A}_t)$ as follows:
  \[
    \bbP(\lambda_{\mathrm{min}}(\mathbf{A}_t) \leq (1 - \varepsilon) \lambda_{\mathrm{min}}(\bbE \mathbf{A}_t)) \leq
    d^2 e^{-\displaystyle\frac{t \varepsilon^2 \lambda_{\mathrm{min}}(\bbE \mathbf{Z}_1) }{2mL}}.
  \]
  If in addition, there exists some $v > 0$, such that $\| \bbE\big[(\mathbf{Z}_1 - \bbE\mathbf{Z}_1)^2 \big] \| \leq v$, then for any $u > 0$, one has
  \[
    \bbP \left( \left\| \mathbf{S}_t \right\| \geq u \right) \leq 2d^{2} e^{\displaystyle-\frac{u^2}{2mLu/3 + 2tv}},
  \]
%   \[
%     \bbP \left( \left\| \mathbf{A}_t - \bbE\mathbf{A}_t \right\| \geq \sqrt{2tvu} +\frac{mLu}{3} \right) \leq d^2 e^{-u}\enspace.
%   \]
\end{prop}

From the second inequality, \cite{rizk2020refined} derived a slightly different inequality that we use here :

\begin{prop}[\citet{rizk2020refined}, Appendix A.3]
  \label{prop:chernoff-matrix_2_bis}
  Let $\mathbf{Z}_1, \ldots, \mathbf{Z}_t$ be $t$ i.i.d.\ random symmetric matrices in $\bbR^{d^2 \times d^2}$ such that there exists $L > 0$ such that $\| \mathbf{Z}_1 \| \leq mL$, almost surely. Let $\mathbf{A}_t \triangleq \sum_{i = 1}^t \mathbf{Z}_i$. Then, for any $u > 0$, one has:
%   \[
%     \bbP \left( \left\| \mathbf{S}_t \right\| \geq u \right) \leq d e^{\displaystyle-\frac{u^2}{2mLu/3 + 2t\sigma}},
%   \]
  \[
    \bbP \left( \left\| \mathbf{A}_t - \bbE\mathbf{A}_t \right\| \geq \sqrt{2tvu} +\frac{mLu}{3} \right) \leq d^2 e^{-u}\enspace.
  \]
  where $v \triangleq \left\| \bbE\big[(\mathbf{Z}_1 - \bbE\mathbf{Z}_1)^2 \big] \right\|$.
\end{prop}

Finally, to prove our main theorem, we need the following lemma.

\begin{lemma}
\label{lemma:bound_cov_norm}
One has $\left\| \Sigma_{\mathcal{Z}}(\lambda^\star)^{-1}\right\| \le \frac{d^{2}}{\nu_{\mathrm{min}}}$, where $\nu_{\mathrm{min}}$
is the smallest eigenvalue of the covariance matrix $\frac{1}{K^{2}}\sum_{z\in\mathcal{Z}}z^{\top}z$.
\end{lemma}

\begin{proof}

Define $\mathcal{B}=\left\{ z\in\mathbb{R}^{d^{2}}: \|z\| =1\right\} $. First, for any semi-definite matrix $\mathbf{A}\in\mathbb{R}^{d^{2}\times d^{2}}$,
we have $ \left\|\mathbf{A}\right\| =\max_{z\in\mathcal{B}} \; z^{\top}\mathbf{A}z$.
Because $\Sigma_{\mathcal{Z}}(\lambda^\star)^{-1}$ is positive definite and symmetric, and by Rayleigh-Ritz theorem,
\begin{align*}
 \left\| \Sigma_{\mathcal{Z}}(\lambda^\star)^{-1}\right\|  & =\max_{z\in\mathcal{B}} \; \frac{z^{\top} \Sigma_{\mathcal{Z}}(\lambda^\star)^{-1} z}{z^\top z} =\max_{z\in\mathcal{B}} \; z^{\top} \Sigma_{\mathcal{Z}}(\lambda^\star)^{-1} z \enspace.
\end{align*}
Let $\mathbf{Z}\in\mathbb{R}^{K^{2}\times d^{2}}$ be the matrix whose
rows are vectors of $\mathcal{Z}$ in an arbitrary order. Notice that $\mathcal{Z}$ spans $\bbR^{d^2}$, since $\mathcal{X}$ spans $\bbR^d$.  Now for
any $z\in\mathcal{B}$, define $\beta^{(z)}\in\mathbb{R}^{K^{2}}$
as a vector such that $z=\mathbf{Z}^{\top}\beta^{(z)}$ . Then,
\begin{align*}
 \left\| \Sigma_{\mathcal{Z}}(\lambda^\star)^{-1}\right\|  & =\max_{z\in\mathcal{B}} \; \beta^{(z)^{\top}}\mathbf{Z} \Sigma_{\mathcal{Z}}(\lambda^\star)^{-1} \mathbf{Z}^{\top}\beta^{(z)}\\
 & =\max_{z\in\mathcal{B}} \; \sum_{i=1}^{d^{2}}\sum_{j=1}^{d^{2}}\beta_{i}^{(z)}\beta_{j}^{(z)}z_{i}^{\top}  \Sigma_{\mathcal{Z}}(\lambda^\star)^{-1} z_{j}\\
 &\le \max_{z\in\mathcal{B}} \; \left\|\beta^{(z)}\right\|_{1}^2 \times\max_{i,j} \; z_{i}^{\top} \Sigma_{\mathcal{Z}}(\lambda^\star)^{-1} z_{j} \enspace.
\end{align*}
Define $\tilde{z}_{i}=\Sigma_{\mathcal{Z}}(\lambda^\star)^{-\frac{1}{2}}z_{i}$. Clearly, $\max_{i,j} \; z_{i}^{\top} \Sigma_{\mathcal{Z}}(\lambda^\star)^{-1}  z_{j} = \max_{i,j} \; \tilde{z}_{i}^{\top} \tilde{z}_{j} = \max_{i} \; \tilde{z}_{i} ^2$.
So we have
\begin{align*}
 \left\| \Sigma_{\mathcal{Z}}(\lambda^\star)^{-1}\right\|  
 &\le \max_{z\in\mathcal{B}} \; \left\|\beta^{(z)}\right\|_{1}^2 \times\max_{z'\in\mathcal{Z}} \; z^{\prime\top} \Sigma_{\mathcal{Z}}(\lambda^\star)^{-1} z'\\
 &\le \max_{z\in\mathcal{B}} \; \left\|\beta^{(z)}\right\|_{1}^2 d^2 \enspace.
\end{align*}
The last inequality comes from Kiefer and Wolfowitz equivalence theorem \citep{KieferEquivalenceThm1960}.
Now observe that $\beta^{(z)}$ can be obtained by least square regression : $\beta^{(z)}=\left(\mathbf{Z}\mathbf{Z}^{\top}\right)^{-1}\mathbf{Z}z=\left(\mathbf{Z}^{\top}\right)^{\dagger}z$
where $\left(\cdot\right)^{\dagger}$ is the Moore-Penrose pseudo-inverse.
Note that $\mathbf{Z}\mathbf{Z}^{\top}$ is a Gram matrix. It is known
that for a matrix having singular values $\left\{ \sigma_{i}\right\}_i $,
its pseudo-inverse has singular values $\begin{cases}
\frac{1}{\sigma_{i}} & \text{if }\sigma_{i}\neq0\\
0 & \text{otherwise}
\end{cases}$ for all $i$. So for $z\in\mathcal{B}$, we have: 
\begin{align*}
 \left\|\beta^{(z)}\right\|_{1}^{2} & \le K^{2} \left\|\beta^{(z)} \right\|_{2}^{2}\le K^{2} \left\|\left(\mathbf{Z}^{\top}\right)^{\dagger}\right\|^{2}\le\frac{K^{2}}{\sigma_{\mathrm{min}}\left(\mathbf{Z}\right)^{2}} \enspace,
\end{align*}
where $\sigma_{\mathrm{min}}\left(\cdot\right)$ refers to the smallest singular value. Let $\nu_{\mathrm{min}}\left(\cdot\right)$ refer to the smallest eigenvalue. Noting that
\[
\sigma_{\mathrm{min}}\left(\mathbf{Z}\right)^{2}=\nu_{\mathrm{min}}\left(\mathbf{Z}^{\top}\mathbf{Z}\right)=K^{2}\nu_{\mathrm{min}}\left(\frac{1}{K^{2}}\sum_{z\in\mathcal{Z}}zz^{\top}\right) \enspace,
\]
yields the desired result.

\end{proof}

We are now ready to state the bound on the random sampling error, relatively to the objective value $\Sigma_{\mathcal{Z}}(\lambda^\star)$ of the convex relaxation solution.
\begin{theorem}
    \label{theorem:convergence_2}
    Let $\lambda^\star \in \mathcal{S}_{\mathcal{Z}}$ be a minimizer of $h_{\mathcal{Z}}$. Let $0 \leq \delta \leq 1$ and let $t_0 > 0$ be such that
    \[
    t_0 = 2L d^2 \log(2d^2/\delta)/\nu_{\mathrm{min}} \enspace,
    \]
    where $L = \max_{z \in \mathcal{Z}} \|z\|^2$ and $\nu_{\mathrm{min}}$ is the smallest eigenvalue of the covariance matrix $\frac{1}{K^{2}}\sum_{z\in\mathcal{Z}}z^{\top}z$.
    Then, at each round $t \geq t_0$, with probability at least $1 - \delta$, the randomized G-allocation strategy for graphical bilinear bandit in Algorithm~\ref{algorithm:Randomized_G_allocation_for_GBB} produces a matrix $\mathbf{A}_t$ such that:
    \begin{align*}
        h_{\mathcal{Z}}(\mathbf{A}_t) \leq (1 + \alpha) h_{\mathcal{Z}}(mt \times \Sigma_{\mathcal{Z}}(\lambda^{\star})) \enspace,
    \end{align*}
    where 
    \begin{align*}
        \alpha = \frac{L d^2}{m\nu_{\mathrm{min}}^2} \sqrt{\frac{2v}{t}\log\left(\frac{2d^2}{\delta}\right)} + o\left(\frac{1}{\sqrt{t}}\right),
    \end{align*}
    and $v \triangleq \| \bbE\big[(\mathbf{A}_1 - \bbE\mathbf{A}_1)^2 \big] \|$.
\end{theorem}
\begin{proof}
    Let $(X^{(1)}_{s})_{s = 1,\dots, t}, \ldots, (X^{(n)}_s)_{s = 1, \ldots, t}$ be $nt$ i.i.d.\ random vectors in $\mathbb{R}^{d}$ such that for all $x \in \mathcal{X}$, $\mathbb{P}\left(X^{(1)}_1 = x\right) = \mu^\star_x$. For $(i, j) \in E$ and $1 \leq s \leq t$, we define the random matrix $\mathbf{Z}^{(i,j)}_{s}$ by 
    \begin{align*}
        \mathbf{Z}^{(i,j)}_{s} = \vect{\left(X^{i}_{s}X^{j\top}_{s}\right)} \vect{\left(X^{i}_{s}X^{j\top}_{s}\right)}^\top \enspace.
    \end{align*}
    Finally, let us define for all $1 \leq s \leq t$, the edge-wise sum $\mathbf{Z}_s \in \mathbb{R}^{d^2 \times d^2}$, that is
    \begin{align*}
    \mathbf{Z}_{s} = \sum_{(i,j)\in E} \mathbf{Z}^{(i,j)}_{s} \enspace.
    \end{align*}
  One can easily notice that $\mathbf{Z}_1, \ldots, \mathbf{Z}_{t}$ are i.i.d.\ random matrices. We define the overall sum $\mathbf{A}_t = \sum_{s=1}^{t} \mathbf{Z}_{s}$ and our goal is to measure how close $f_{\mathcal{Z}}(\mathbf{A}_t)$ is to $f_{\mathcal{Z}}(mt \times \Sigma_{\mathcal{Z}}(\lambda^{\star}))$, where $mt$ corresponds to the total number of sampled arms $z \in \mathcal{Z}$ during the $t$ rounds of the learning procedure. By definition of $\mathbf{A}_t$, one has
  \begin{align*}
      \max_{z \in \mathcal{Z}} \; z^\top \left(\bbE \mathbf{A}_t\right)^{-1} z &= \max_{z \in \mathcal{Z}} \; z^\top \left( \sum_{s = 1}^t \sum_{(i, j) \in E} \bbE\left[\mathbf{Z}_s^{(i, j)}\right] \right)^{-1} z\\
      &= \max_{z \in \mathcal{Z}} \; z^\top \left( \sum_{s = 1}^t \sum_{(i, j) \in E} \sum_{x, x' \in \mathcal{X}} \mu^\star_x \mu^\star_{x'} \vect{(x x^{\prime\top})}\vect{(x x^{\prime\top})}^\top \right)^{-1} z\\
      &= \max_{z \in \mathcal{Z}} \; z^\top \left( \sum_{s = 1}^t \sum_{(i, j) \in E} \sum_{z^\prime \in \mathcal{Z}} \lambda^\star_{z^\prime} z^\prime z^{\prime\top} \right)^{-1} z\\
      &= f_{\mathcal{Z}}(mt \Sigma_{\mathcal{Z}}(\lambda^\star)) \enspace.
  \end{align*}
    This allows us to bound the relative error as follows:
    \begin{align*}
        \alpha &= \frac{f_{\mathcal{Z}}(\mathbf{A}_t)}{f_{\mathcal{Z}}(mt \times \Sigma_{\mathcal{Z}}(\lambda^{\star}))} - 1 \\
        &= \frac{\max_{z\in\mathcal{Z}} \; z^\top\left(\mathbf{A}_t^{-1} - \left(\mathbb{E}\mathbf{A}_t\right)^{-1} + \left(\bbE\mathbf{A}_t\right)^{-1}\right)z}{f_{\mathcal{Z}}(mt \times \Sigma_{\mathcal{Z}}(\lambda^{\star}))} - 1 \\
        &\leq \frac{\max_{z\in\mathcal{Z}} \; z^\top\left(\mathbf{A}_t^{-1} - \left(\mathbb{E}\mathbf{A}_t\right)^{-1}\right)z}{f_{\mathcal{Z}}(mt \times \Sigma_{\mathcal{Z}}(\lambda^{\star}))} \enspace.
    \end{align*}
    Using the fact that $f_{\mathcal{Z}}(mt \Sigma_{\mathcal{Z}}(\lambda^{\star})) = d^{2}/mt$ \citep{KieferEquivalenceThm1960}, we obtain
    \begin{align*}
       \alpha &\leq \frac{mt}{d^2} \times \max_{z\in\mathcal{Z}} \; z^\top\left(\mathbf{A}_t^{-1} - \left(\mathbb{E}\mathbf{A}_t\right)^{-1}\right)z \\
        &\leq \frac{mt}{d^2} \times \max_{z\in\mathcal{Z}} \; \|z\|^2 \|\mathbf{A}_t^{-1} - \left(\mathbb{E}\mathbf{A}_t\right)^{-1}\| \\
        &\leq \frac{mtL}{d^2} \times \|\mathbf{A}_t^{-1} - \left(\mathbb{E}\mathbf{A}_t\right)^{-1}\| \enspace.
    \end{align*}
    Therefore, controlling the quantity $\|\mathbf{A}_t^{-1} - \left(\mathbb{E}\mathbf{A}_t\right)^{-1}\|$ will allow us to provide an upper bound on the relative error. Notice that 
    \begin{align*}
        \|\mathbf{A}_t^{-1} - \left(\mathbb{E}\mathbf{A}_t\right)^{-1}\| &= \|\mathbf{A}_t^{-1}\left( \mathbb{E}\mathbf{A}_t - \mathbf{A}_t\right) \left(\mathbb{E}\mathbf{A}_t\right)^{-1}\| \\
        &\leq \|\mathbf{A}_t^{-1}\| \ \| \mathbb{E}\mathbf{A}_t - \mathbf{A}_t\| \  \|\left(\mathbb{E}\mathbf{A}_t\right)^{-1}\| \enspace.
    \end{align*}
    % Let us assume that the batch size $n$ satisfies:
    % \[
    %   n > \frac{2L \log d}{\lambda_{\mathrm{min}}(\bbE\mathbf{X}_1)}.
    % \]
    Using Proposition~\ref{prop:chernoff-matrix_2}, we know that for any $d^{2}e^{-\frac{t\lambda_{\mathrm{min}}(\mathbb{E}\mathbf{Z}_1)}{mL}}< \delta_h < 1$, the following holds:
    \begin{align*}
     \| \mathbf{A}_t^{-1} \| \leq \frac{\| \left(\mathbb{E}\mathbf{A}_t\right)^{-1} \|}{1 - \sqrt{\frac{2mL}{t} \| \left(\mathbb{E}\mathbf{Z}_1\right)^{-1} \| \log(d^2/\delta_h)}} \enspace,
    \end{align*}
    with probability at least $1 - \delta_h$. Similarly, using Proposition~\ref{prop:chernoff-matrix_2_bis}, for any $0 < \delta_b < 1$, we have
    \[
      \|\mathbf{A}_t - \mathbb{E}\mathbf{A}_t\| \leq \frac{mL}{3}\log\frac{d^2}{\delta_b} + \sqrt{2t v^2\log\frac{d^2}{\delta_b}} \enspace,
    \]
    with probability at least $1 - \delta_b$. Combining these two results with a union bound leads to the following bound, with probability $1 - (\delta_b + \delta_h)$:
    \begin{align*}
      \left\|\mathbf{A}_t^{-1} - \left(\mathbb{E}\mathbf{A}_t\right)^{-1} \right\| \leq \left\| \left(\mathbb{E}\mathbf{A}_t\right)^{-1} \right\|^2\frac{(mL/3)\log(d^2/\delta_b) + \sqrt{2tv\log(d^2/\delta_b)}}{1 - \sqrt{(2mL/t) \left\| \left(\mathbb{E}\mathbf{Z}_1\right)^{-1} \right\| \log(d^2/\delta_h)}} \enspace.
    \end{align*}
    In order to obtain a unified bound depending on one confidence parameter $1 - \delta$, one could optimize over $\delta_b$ and $\delta_h$, subject to $\delta_b + \delta_h = \delta$. This leads to a messy result and a negligible improvement. One can use simple values $\delta_b = \delta_h = \delta / 2$, so the overall bound becomes, with probability $1 - \delta$:
    \begin{align*}
      \|\mathbf{A}_t^{-1} - \left(\mathbb{E}\mathbf{A}_t\right)^{-1}\|
      \leq \frac{1}{t m^2} \left\|\Sigma_{\mathcal{Z}}(\lambda^\star)^{-1}\right\|^2\sqrt{\frac{2v}{t}\log\left(\frac{2d^2}{\delta}\right)}
      \left(\frac{1 + \sqrt{ \frac{m^2L^2\log(2d^2/\delta)}{18v t}}}{1 - \sqrt{\frac{2L\|\Sigma_{\mathcal{Z}} (\lambda^\star)^{-1}\| \log(2d^2/\delta)}{t}}}\right) \enspace.
    \end{align*}
    This can finally be formulated as follows:
    \begin{align*}
      \left\|\mathbf{A}_t^{-1} - \left(\mathbb{E}\mathbf{A}_t\right)^{-1}\right\|
      \leq \frac{1}{tm^2}\left\|\Sigma_{\mathcal{Z}}(\lambda^\star)^{-1}\right\|^2\sqrt{\frac{2v}{t}\log\left(\frac{2d^2}{\delta}\right)} + o\left(\frac{1}{t\sqrt{t}}\right)\enspace.
    \end{align*}
    Using the obtained bound on $\|\mathbf{A}_t^{-1} - \mathbb{E}(\mathbf{A}_t)^{-1}\|$ yields
    \begin{align*}
        \frac{f_{\mathcal{Z}}(\mathbf{A}_t)}{f_{\mathcal{Z}}(mt \times \Sigma_{\mathcal{Z}}(\lambda^{\star}))} - 1 &\leq \frac{mtL}{d^2} \times \left(\frac{1}{tm^2}\left\|\Sigma_{\mathcal{Z}}(\lambda^{\star})^{-1}\right\|^2\sqrt{\frac{2v}{t}\log\left(\frac{2d^2}{\delta}\right)} + o\left(\frac{1}{t\sqrt{t}}\right)\right) \\
        &\leq \frac{L}{md^2} \left\|\Sigma_{\mathcal{Z}}(\lambda^{\star})^{-1}\right\|^2\sqrt{\frac{2v}{t}\log\left(\frac{2d^2}{\delta}\right)} + o\left(\frac{1}{\sqrt{t}}\right)\enspace,
    \end{align*}
    By noticing that $f_{\mathcal{Z}}(mt \times \Sigma_{\mathcal{Z}}(\lambda^{\star})) \leq f_{\mathcal{Z}}(\mathbf{A}^{\star}_t)$ and by using Lemma \ref{lemma:bound_cov_norm}, the result holds.
\end{proof}

\section{Variance analysis}
\label{sec:supp_variance}
\paragraph{Star graph.}
The covariance matrix of the star graph can be bounded as follows:
\begin{align*}
    \var(\mathbf{A}_1) \preceq m \times P \cdot \mathbf{I} + (n_{\mathrm{S}} - 1)(n_{\mathrm{S}} - 2) M \cdot \mathbf{I} + n_{\mathrm{S}}(n_{\mathrm{S}} - 1) N \cdot \mathbf{I} \enspace.
\end{align*}
Since the star graph of $m$ edges has a number of nodes $n_{\mathrm{S}} = m/2 + 1$, we have
\begin{align*}
    \|\var(\mathbf{A}_1)\| \leq m \times P + (M + N) \times O\left(m^2\right)  \enspace.
\end{align*}

\paragraph{Complete graph.} As for the star graph,
\begin{align*}
    \var(\mathbf{A}_1) \preceq m \times P\cdot \mathbf{I} + n_{\mathrm{Co}}(n_{\mathrm{Co}}-1)(n_{\mathrm{Co}}-2) M\cdot \mathbf{I}+ n_{\mathrm{Co}}(n_{\mathrm{Co}}-1)(n_{\mathrm{Co}}-1) N \cdot \mathbf{I} \enspace.
\end{align*}
Since the complete graph of $m$ edges has a number of nodes $n_{\mathrm{Co}} = \left(1 + \sqrt{4m+1}\right)/2$, we have
\begin{align*}
    \|\var(\mathbf{A}_1)\| \leq m \times P + (M + N) \times O\left(m\sqrt{m}\right) \enspace.
\end{align*}

\paragraph{Circle graph.}
Again,
\begin{align*}
    \var(\mathbf{A}_1) \preceq m \times P \cdot \mathbf{I}  + 2n_{\mathrm{Ci}} M \cdot \mathbf{I} + 4n_{\mathrm{Ci}} N \cdot \mathbf{I} \enspace.
\end{align*}
Since the circle graph of $m$ edges has a number of nodes $n_{\mathrm{Ci}} = m/2$, we have
\begin{align*}
    \|\var(\mathbf{Z}_1)\| \leq m \times P + (M + N) \times O\left(m\right)  \enspace.
\end{align*}

\paragraph{Matching graph.}
Finally,
\begin{align*}
    \var(\mathbf{A}_1) \preceq m \times P \cdot \mathbf{I}  + n_{\mathrm{M}} N \cdot \mathbf{I} \enspace.
\end{align*}
Since the matching graph of $m$ edges has a number of nodes $n_{\mathrm{M}} = m$, we have
\begin{align*}
    \|\var(\mathbf{A}_1)\| \leq m \times P + m \times N \enspace.
\end{align*}

\section{Generalization}
\label{sec:supp_generalization}

In this section, we provide some insights into the generalization to broader reward settings.

\subsection{When \texorpdfstring{$\mathbf{M}_\star$}{M} is not symmetric}

Consider the same graphical bilinear bandit setting as the one explained in the paper with the only difference that $\mathbf{M}_\star$ is not symmetric. We recall here that in the graph $\mathcal{G}=(V,E)$ associated to the graphical bilinear bandit setting, $(i,j) \in E$ if and only if $(j,i) \in E$. Hence, for a given allocation $(x^{(1)}, \dots, x^{(n)}) \in \mathcal{X}^n$, one can write the associated expected global reward as follows : 

\begin{align*}
    \sum_{(i,j) \in E} x^{(i)\top} \mathbf{M}_\star x^{(j)} & = \sum_{i = 1}^{n} \sum_{\substack{j \in \mathcal{N}(i)\\ j > i}} x^{(i)\top} \mathbf{M}_\star x^{(j)} + x^{(j)\top} \mathbf{M}_\star x^{(i)} \\
    &= \sum_{i = 1}^{n} \sum_{\substack{j \in \mathcal{N}(i)\\ j > i}} x^{(i)\top} \mathbf{M}_\star x^{(j)} + \left(x^{(j)\top} \mathbf{M}_\star x^{(i)}\right)^\top \\
    &= \sum_{i = 1}^{n} \sum_{\substack{j \in \mathcal{N}(i)\\ j > i}} x^{(i)\top} \mathbf{M}_\star x^{(j)} + x^{(i)\top} \mathbf{M}_\star^\top x^{(j)} \\
    &= \sum_{i = 1}^{n} \sum_{\substack{j \in \mathcal{N}(i)\\ j > i}} x^{(i)\top} \left(\mathbf{M}_\star x^{(j)} + \mathbf{M}_\star^\top x^{(j)}\right) \\
    &= \sum_{i = 1}^{n} \sum_{\substack{j \in \mathcal{N}(i)\\ j > i}} x^{(i)\top} \left(\mathbf{M}_\star + \mathbf{M}_\star^\top\right) x^{(j)} \enspace.
\end{align*}
Let us denote $\bar{\mathbf{M}}_\star = \mathbf{M}_\star + \mathbf{M}_\star^\top$. One can notice that $\bar{\mathbf{M}}_\star$ is symmetric.
Solving the graphical bilinear bandit with the matrix $\bar{\mathbf{M}}_\star$ is exactly what we propose throughout the main paper. 

\subsection{When the reward captures more information than the interactions between agents}

Consider the real world problems introduced in the paper, but with the difference that instead of a reward only related to the interaction between two neighboring agents/nodes, there is an additional term that informs about the absolute quality of the arm chosen by the agent itself. More formally we consider the following reward $r^{(i,j)}_t$ for the node $i$:
\begin{align*}
r^{(i,j)}_{t} = x^{(i)\top}_t \mathbf{M}_\star x^{(j)}_t + x^{(i)\top}_t \beta_\star + \eta^{(i,j)}_t \enspace. \label{eq:reward_2}
\end{align*}
where $\beta_\star \in \mathbb{R}^d$ is a second unknown parameter that allows to capture the quality of the arm chosen by the node $i$ independently of its neighbors.

In order to add a constant term in the reward, let us construct the set $\Tilde{\mathcal{X}} \subset \mathbb{R}^{d+1} $ such that each arm $x \in \mathcal{X}$ is associated to a new arm $\Tilde{x} \in \Tilde{\mathcal{X}}$ defined as $\Tilde{x}^\top = (x^\top, 1)$.
% \begin{align*}
%     \Tilde{x} =  \begin{pmatrix}
%         \left[\begin{matrix} \\ x \\ \\ \end{matrix}\right] \\[0.6cm]
%         1
%     \end{pmatrix}
% \end{align*}
Moreover, let us define the matrix $\Tilde{\mathbf{M}}^\star \in \mathbb{R}^{(d+1) \times (d+1)}$ as follows:

\begin{align*}
    \Tilde{\mathbf{M}}_\star = \begin{pmatrix}
        \left[\begin{matrix} & & \\ & \mathbf{M}_\star & \\ & & \end{matrix}\right] \left[\begin{matrix} \\ \beta_\star \\ \\ \end{matrix}\right] \\[0.6cm]
        \left[\begin{matrix} 0 \ \ \ & \cdots & \ \  \ 0 \end{matrix}\right]
    \end{pmatrix} \enspace.
\end{align*}

One can easily verify that for any edge $(i,j) \in E$ and any time step $t$, the reward $r^{(i,j)}_{t}$ can now be written as follows:
\begin{align*}
r^{(i,j)}_{t} &= \Tilde{x}^{(i)\top}_t \Tilde{\mathbf{M}}_{\star} \Tilde{x}^{(j)}_t + \eta^{(i,j)}_t \enspace,
\end{align*}
which leads to the same graphical bilinear bandit setting explained in Section \ref{sec:preliminaries}, this time in dimension $d+1$ instead of $d$. Hence, all the previous results hold for this more general graphical bilinear bandit problem, provided any dependence in $d$ is modified to $d + 1$.

\section{Computing \texorpdfstring{$\mu^\star$}{mu}}
\label{sec:supp_experiments}

In Algorithm~\ref{algorithm:Randomized_G_allocation_for_GBB}, we need to find the solution $\mu_\star$ of $\min_{\mu \in \mathcal{S}_{\mathcal{X}}} h_{\mathcal{X}}(\mu)$. In fact we need $\mu_\star$ to sample from it.
We show that for any set $X$, the function $h_X$ is convex and we use the Frank-Wolfe algorithm \citep{frank1956algorithm} to compute $\mu_\star$ and $\lambda_\star$. The convergence of the algorithm has been proven in \citet{damla2008linear}. Note that one can only compute $\mu_\star$ or $\lambda_\star$ to obtain the other one thanks to \ref{theorem:mu_to_lambda_2}.

\begin{prop}
\label{f_is_convex}
Let $d>0$, for any set $X \subset \mathbb{R}^{d}$, $h_X$ is convex.
\end{prop}

\begin{proof}
Let $(\lambda, \lambda^\prime) \in \mathcal{S}_{X}^2$ be two distributions in $\mathcal{S}_{X}$.
If either $\Sigma_X(\lambda)$ or $\Sigma_X(\lambda^\prime)$ are not invertible, then for any $t \in [0, 1]$ one has 
\[
h_X(t \lambda + (1 - t) \lambda^\prime) \leq t h_X(\lambda) + (1 - t) h_X(\lambda^\prime) = +\infty \enspace.
\]
Otherwise, for $t \in [0,1]$, we define the positive definite matrix $\mathbf{Z}(t) \in \bbR^{d \times d}$ as follows:
\[
\mathbf{Z}(t) = t \Sigma_X(\lambda)  + (1 - t) \Sigma_X(\lambda^\prime) \enspace.
\]
Simple linear algebra \citep{petersen2012matrix} yields
\[
    \frac{\partial \mathbf{Z}(t)^{-1}}{\partial t} = \mathbf{Z}(t)^{-1} \frac{\partial \mathbf{Z}(t)}{\partial t} \mathbf{Z}(t)^{-1} \enspace.
\]
Using this result and the fact that $\partial^2 \mathbf{Z}(t) / \partial t^2 = 0$, we obtain
\[
    \frac{\partial^2 \mathbf{Z}(t)^{-1}}{\partial t^2} = 2 \mathbf{Z}(t)^{-1} \frac{\partial \mathbf{Z}(t)}{\partial t} \mathbf{Z}(t)^{-1} \frac{\partial \mathbf{Z}(t)}{\partial t} \mathbf{Z}(t)^{-1} \enspace.
\]
%\begin{align*}
%     \mathbf{Z}(t) = \sum_{x \in X} \left(t \times \lambda_x + (1-t) \times \lambda_x^{\prime}\right) x x^{\top} = t \sum_{x \in X}  \lambda_x x x^{\top} + (1-t) \sum_{x \in X}  \lambda_x^\prime x x^{\top} \enspace.
% \end{align*}
%
% Notice that for any $t \in [0, 1]$, $\mathbf{Z}(t)$ is positive definite, since $X$ spans $\bbR^d$.
%
% \begin{align*}
%     \left(ZZ^{-1}\right) = I_{K^2} & \implies \frac{\partial \left(ZZ^{-1}\right)}{\partial t}  = \frac{\partial \left(I_{K^2}\right)}{\partial t} \\
%     & \implies \frac{\partial Z}{\partial t} Z^{-1} + Z  \frac{\partial Z}{\partial t}  = 0_{K^2} \\
%     & \implies \frac{\partial Z^{-1}}{\partial t}  = - Z^{-1} \frac{\partial Z}{\partial t}  Z^{-1}
% \end{align*}
%
% Moreover,
% \begin{align*}
%     \frac{\partial^{2} Z^{-1}}{\partial t^2} &= -\left(- Z^{-1} \frac{\partial Z}{\partial t} Z^{-1}  \frac{\partial Z}{\partial t} Z^{-1} + Z^{-1}  \left(\frac{\partial^{2} Z}{\partial t^2} Z^{-1} + \frac{\partial Z}{\partial t}  \left(- Z^{-1}  \frac{\partial Z}{\partial t}  Z^{-1} \right) \right) \right)\\
%     &= -\left(- Z^{-1} \frac{\partial Z}{\partial t} Z^{-1}  \frac{\partial Z}{\partial t} Z^{-1} - Z^{-1}  \frac{\partial Z}{\partial t} Z^{-1}  \frac{\partial Z}{\partial t}  Z^{-1} \right) \ \ \ \  \text{(because $\frac{\partial^{2} Z}{\partial t^2} = 0$)}\\
%     &= 2  Z^{-1} \frac{\partial Z}{\partial t} Z^{-1} \frac{\partial Z}{\partial t} Z^{-1}
% \end{align*}
Therefore, for any $x \in X$,
\begin{align*}
    \frac{\partial^2 x^\top \mathbf{Z}(t)^{-1} x}{\partial t^2}
    &= 2 x^\top \mathbf{Z}(t)^{-1} \frac{\partial \mathbf{Z}(t)}{\partial t} \mathbf{Z}(t)^{-1} \frac{\partial \mathbf{Z}(t)}{\partial t} \mathbf{Z}(t)^{-1} x\\
    &= 2 \left( \frac{\partial \mathbf{Z}(t)}{\partial t} \mathbf{Z}(t)^{-1} x\right)^\top  \mathbf{Z}(t)^{-1} \left( \frac{\partial \mathbf{Z}(t)}{\partial t} \mathbf{Z}(t)^{-1} x\right)\\
    &\geq 0 \enspace,
\end{align*}
which shows convexity for any fixed $x \in X$. The final results yields from the fact that $h_X$ is a maximum over convex functions.

% For any, $x \in X$, we have : 
% \begin{align*}
%     \frac{\partial^{2} \left(x^\top Z^{-1} s \right)}{\partial t^2} &= 2 s^\top  Z^{-1} \frac{\partial Z}{\partial t} Z^{-1} \frac{\partial Z}{\partial t} Z^{-1}x \\
%     &=  x^\top  \left(\frac{\partial Z}{\partial t} Z^{-1}\right)^\top Z^{-1} \left(\frac{\partial Z}{\partial t} Z^{-1}\right)x  \\
%     &=  \left(\frac{\partial Z}{\partial t} Z^{-1} x\right)^\top Z^{-1} \left(\frac{\partial Z}{\partial t} Z^{-1}x\right) \\
%     &\geq 0  \ \ \ \ \text{because $Z^{-1}$ is positive semi-definite}
% \end{align*}

% Hence, $\forall x \in X$, the functions $g_{x}(t) = x^{\top} Z^{-1} x$ are convex with respect to $t \in [0,1]$.

% We recall that a function is convex if and only if its epigraph is convex. Moreover $h_X(\lambda + (1-t) \times \lambda^\prime) = \max_{x \in S} g_{x}(t)$, thus taking the maximum value of convex functions $g_{x}$ gives a function $h_X$ with an epigraph equals to the intersection of the epigraphs of the functions $g_{x}$, which is the intersection of convex sets that gives a convex set. Hence $h_X$ is convex with respect to $t \in [0,1]$.

\end{proof}

\section{Additional experiment and information}

We define the set of arms $\mathcal{X} \subset \mathbb{R}^5$ that is made of $|\mathcal{X}| =  100$ node-arms randomly sampled from a multivariate 5-dimensional Gaussian distribution $\mathcal{N}(0, I)$ and then normalized so that $\|x\|=1$ for all $x \in \mathcal{X}$. In all the figures the results are averaged over 100 random repetitions of the experiments.

We propose to validate our insight and compute the evolution of $\|\var\left(\mathbf{A}_1\right)\|$ for the three types of graphs (star, complete and circle) and different number of edges. The results are shown in Figure~\ref{fig:var}. One can notice that we retrieve the $O(m^2)$ dependence of the variance for the star graph, the $O(m \sqrt{m})$ for the complete graph and the linear dependence $O(m)$ for the circle graph.

\begin{figure}[H]
    \centering
    \includegraphics{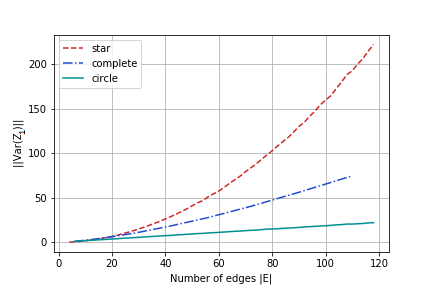}
    \caption{Evolution of the variance according to the number of edges and the type of graph (star, complete, circle), the variance being averaged over 100 repetitions.}
    \label{fig:var}
\end{figure}

\paragraph{Machine used for all the experiments.}  Intel(R) Xeon(R) CPU E5-2667 v4 @ 3.20GHz - 24 CPUs used.

\end{document}